%% file: main.tex
\documentclass[preprint,review,12pt]{elsarticle}

\usepackage{bbold}
\usepackage{amssymb}
\usepackage{amsmath}
\usepackage{amsthm}
\usepackage{subcaption}
\usepackage{tikz}
\usetikzlibrary{calc,arrows,shapes,automata,backgrounds,petri}

\newtheorem{theorem}{Theorem}

\newtheorem{assumption}{Assumption}

\input{commands}

\journal{Computers \& Mathematics with Applications}

\begin{document}
\begin{frontmatter}

\title{Planar Curve Registration using Bayesian Inversion}

\author[dtu]{Andreas Bock\corref{cor1}}
\ead{aasbo@dtu.dk}
\affiliation[dtu]{organization={Department of Applied Mathematics and Computer Science, Technical University of Denmark},
             addressline={Richard Petersens Plads, Building 324},
             city={Kongens Lyngby},
             postcode={2800},
             country={Denmark}}
\cortext[cor1]{Corresponding author}
\author[imperial]{Colin J. Cotter}
\ead{colin.cotter@imperial.ac.uk}
\affiliation[imperial]{organization={Department of Mathematics, Imperial College London},
             addressline={180 Queen's Gate, South Kensington},
             city={London},
             postcode={SW72RH},
             country={United Kingdom}}
\author[baylor]{Robert C.~Kirby}
\ead{robert_kirby@baylor.edu}
\affiliation[baylor]{organization={Department of Mathematics, Baylor University},
             addressline={1410 S.4th Street, Sid Richardson Science Building},
             city={Waco},
             postcode={76706},
             state={Texas},
             country={United States of America}}

\input{abstract}

\begin{keyword}
Closed curve matching \sep Nonconforming finite element method \sep Bayesian inverse problem
\PACS
87.57.N
\MSC
65M60   
\sep 65P10  
\sep 65M32  
\end{keyword}

\end{frontmatter}

\input{introduction}
\input{diffeo}

\input{discretisation}
\input{inverse_problem}
\input{numerical_results}
\input{conclusion}
\appendix
\input{appendix}

\bibliographystyle{elsarticle-num} 
\bibliography{main.bib}

\end{document}

%% file: commands.tex
\newcommand{\idop}{\textbf{id}}

\newcommand{\diff}{\,\text{d}}
\newcommand{\dth}{\diff{\theta}}
\newcommand{\half}{{\frac 12}}

\newcommand{\Rd}{\mathbb{R}^d}

\newcommand{\transp}{^\top}
\newcommand{\invtransp}{^{-\top}}

\newcommand{\inv}{^{-1}}

\newcommand{\bfalpha}{{\boldsymbol{\alpha}}}

\newcommand{\lipsobo}{{W^{1,\infty}(\Omega)^d}}

\newcommand{\meshskel}{\mathcal{E}_{h}}
\newcommand{\meshskelint}{\mathring{\mathcal{E}}_{h}}

\newcommand{\kalmanp}{\mathfrak{K}_{\mathbf P}}

\newcommand{\cov}{\text{Cov}}

\newcommand{\diffomega}{\text{Diff}(\Omega)}

\newcommand{\diffplusS}{\text{Diff}_+(S^1)}

\newcommand{\emb}{\text{Emb}(S^1,\,\mathbb{R}^d)}
\newcommand{\embomega}{\text{Emb}(S^1,\,\Omega)}

\newcommand{\momentum}{\mathbf{p}}

\newcommand{\AvgMomentum}{\bar{\momentum}}

\newcommand{\bq}{\boldsymbol{\tau}}
\newcommand{\AvgPoint}{\bar{\bq}}
\newcommand{\qtarget}{\bq_\textnormal{target}}
\newcommand{\bqtarget}{\bq_\textnormal{target}}

\newcommand{\MomentumSpace}{\mathbf{P}}

\DeclareUnicodeCharacter{0301}{\'{e}}

\makeatletter
\newsavebox{\@brx}
\newcommand{\llangle}[1][]{\savebox{\@brx}{\(\m@th{#1\langle}\)}%
  \mathopen{\copy\@brx\kern-0.5\wd\@brx\usebox{\@brx}}}
\newcommand{\rrangle}[1][]{\savebox{\@brx}{\(\m@th{#1\rangle}\)}%
  \mathclose{\copy\@brx\kern-0.5\wd\@brx\usebox{\@brx}}}
\makeatother

\newcommand{\bfv}{\mathbf{v}}
\def \mcp{\mathcal{P}}
\def \bfe{\mathbf{e}}
\def \bfn{\mathbf{n}}
\def \bfs{\mathbf{s}}
\def \bft{\mathbf{t}}
\def \bfv{\mathbf{v}}
\def \bfx{\mathbf{x}}
\def \bfy{\mathbf{y}}

\def \wx{\mathcal{W}}
\def \wxr{\mathcal{W}_r}

%% file: abstract.tex
\begin{abstract}
We study parameterisation-independent closed planar curve matching
as a Bayesian inverse problem. The motion of the curve is modelled 
via a curve on the diffeomorphism group acting on the ambient space,
leading to a \emph{large deformation diffeomorphic metric mapping} 
(LDDMM) functional penalising the kinetic energy of the deformation.
We solve Hamilton's equations for the curve matching problem using
the Wu-Xu element  [S. Wu, J. Xu, Nonconforming finite element spaces
for $2m^{\text{th}}$ order partial differential equations on $\mathbb{R}^n$ simplicial 
grids when $m = n + 1$, Mathematics of Computation 88 (316) (2019)
 531–551] which provides 
mesh-independent Lipschitz constants for the forward motion of the
curve, and solve the inverse problem for the momentum using Bayesian
inversion.  Since this element is not affine-equivalent
we provide a \emph{pullback theory} which expedites the implementation
and efficiency of the forward map. We adopt ensemble Kalman inversion using
a negative Sobolev norm mismatch penalty to measure the discrepancy
between the target and the ensemble mean shape. We provide
several numerical examples to validate the approach.
\end{abstract}

%% file: introduction.tex
\section{Introduction}

Closed curve matching is a central problem in shape analysis where the goal is
to bring into alignment two closed curves in $\emb$ called the \emph{template} 
and the \emph{target} \cite{younes2010shapes}. For unparameterised curves,
the shape space for these objects is 
$Q = \emb\setminus\diffplusS$ \cite{michor2007overview,bauer2014overview}.
This quotient space disassociates the curve from arbitrary reparameterisation
since they do not affect the range of the curves in question.
This gives rise to studying the commuting left and right actions of two Lie groups,
$G=\text{Diff}_+(\mathbb{R}^2)$ and  $H=\text{Diff}_+(S^1)$ as in  \cite{cotter2009geodesic}:
\begin{equation}\label{eq:group_actions}
GQ = \text{Emb}(S^1, G.\mathbb{R}^2),\qquad HQ = \text{Emb}(H.S^1,\mathbb{R}^2).
\end{equation}
In the context of developing algorithms for planar curve matching, these group
actions must be explicitly discretised. In this paper we our shape space with the
so-called \emph{outer} metric inherited by $G$ which acts on the ambient space.
This is in contrast to inner metrics intrinsically defined on the embedded shape
\cite{bauer2011new}, see  \cite{michor2007overview} for a comparison. To treat
the parameterisation, one can parameterise elements of $H$ using its Lie
algebra and exploit its vector space structure. In this paper we consider a mismatch
penalty that eliminates the need to treat $H$ explicitly. Instead we note that two
closed curves $c_1$ and $c_2$ are similar when the difference between the indicator 
function $\mathbb{1}$ evaluated on their interiors is small. For some linear differential
operator $\mathcal{C}$ we therefore we define the mismatch, or \emph{misfit}, between
them as:
\begin{equation}\label{eq:mismatch}
\mathfrak{E}(c_1, c_2) = \| \mathbb{1}_{c_1} -  \mathbb{1}_{c_2} \|_{\mathcal{C}}^2,
\end{equation}
where $\| f \|_{\mathcal{C}}^2 = \langle\mathcal{C}^{-1}f,\mathcal{C}^{-1}f\rangle_{L^2}$ 
over some computational domain described later. For the outer metric we take the LDDMM
approach \cite{glaunes2008large} and consider a one-parameter family of velocities
$t\mapsto u_t$ encoding the motion of the ambient space (and therefore the 
shape) which simultaneously provides a distance measure.\\

We discretise the velocity field using finite elements, specifically the Wu-Xu element
\cite{wu2019nonconforming}. This element provides a nonconforming discretisation for
sixth order operators; sixth order is necessary for the diffeomorphism to be sufficiently
smooth for the computations that we undertake. The implementation of this element in
Firedrake \cite{rathgeber2016firedrake} is made possible by applying the theory of 
\cite{kirby2018general} and techniques for code generation in~\cite{kirby2019code}. Given certain assumptions on the structure of our 
problem we can identify this entire family of velocities with a single
initial momentum defined as a function over the template. We eliminate its
evolution equation by using the analytical solution, and restrict the initial
conditions to only generate geodesics in the space of unparameterised curves.
This results in a \emph{forward map}, taking as input the momentum and 
providing the diffeomorphism whose action maps the template to 
the target curve. After obtaining a finite element discretisation
of this map we apply massively parallel and derivative-free ensemble Kalman
inversion which we use to invert the forward map for the initial momentum 
determining the geodesic motion of the curve.

\subsection{Previous work}

Diffeomorphic registration has enjoyed a rich literature since the
seminal works \cite{grenander1998computational,dupuis1998variational}.
For curves specifically, \cite{glaunes2004diffeomorphic,vaillant2005surface}
present the first algorithms for modelling curve matching via gradient
descent methods. \cite{glaunes2008large} represents curves as measures onto
which a Hilbert structure is endowed, and computations of both the outer
metric and the curves are done via radial reproducing kernels producing
$C^\infty$ velocities. In particular, curves were represented as geometric
currents.
\cite{bauer2019relaxed} studies such a varifold-based loss function
for elastic metrics, see also \cite{bauer2015curve,bauer2017numerical,hartman2022elastic}
for numerical frameworks for $H^2$ metrics. \cite{bharath2020analysis}
contains a review of methods related to elastic curves.\\

In this paper we are concerned with higher-order metrics using
finite elements. While there is typically a loss of regularity incurred
by these methods, they offer more computationally efficient methods than
e.g. kernel methods. Finite elements also benefit from spatial 
adaptivity allowing for local refinement e.g. close to embedded curves.
Closest to our approach in terms of discretisation are
\cite{cotter2008variational,cotter2012reparameterisation} where a 
\emph{particle-mesh} method is employed for curve matching where the
curve was discretised into a finite set of particles, acted on by an
outer metric. However, we consider instead an outer metric \emph{finite element}
discretisation (as opposed to the intrinsic metric in 
\cite{bauer2011new}). \cite{gunther2013flexible} presents an adaptive
Eulerian FEM discretisation of the velocity field for LDDMM using $C^1$
cubic Hermite elements and compares the deformations generated using
$C^\infty$ fields to assess the effect of the loss of regularity.
Smooth mesh deformations are also of interest in shape optimisation 
where the aim is to transform a mesh such that some functional is minimised.
Finite element methods are also adopted here, with deformation fields being
discretised using B-splines \cite{hollig2003finite}, harmonic polynomials 
or Lagrange finite elements depending the desired resolution or order 
\cite{paganini2018higher}.  Using the finite element space introduced
in \cite{wu2019nonconforming} we can guarantee that the Lipschitz
norm remains bounded under mesh refinement without resorting to spline
or kernel discretisations. As mentioned, we use Firedrake 
\cite{rathgeber2016firedrake} for all our numerical experiments, see 
also \cite{paganini2021fireshape} for an extension of this package for
shape optimisation.\\

Our formulation eliminates the need to integrate the momentum equation 
via its analytical solution thereby improving on the typically larger cost
of Hamiltonian shooting based methods \cite{vialard2012diffeomorphic} 
compared to an LDDMM formulation \cite{glaunes2008large}. We
only need to solve an elliptic equation to obtain the velocity
and use a simple variational Euler scheme to
evolve the diffeomorphism. Traditional approaches in numerical shape
analysis often apply a shooting procedures to determine the initial
momentum transporting the image or landmarks to the desiderata, see
e.g. \cite{bock2019selective,miller2006geodesic}. Bayesian approaches
have been employed before in the context of shape analysis, see e.g.
\cite{cotter2013bayesian} where function space Markov Chain Monte Carlo
is used to characterise the posterior density of momenta generating a
given shape. Similar to our approach is \cite{bock2021learning}
in which  ensemble Kalman inversion
\cite{iglesias2013ensemble,iglesias2016regularizing} is applied to
recover the momentum for landmark matching.

\subsection{Organisation}

Section \ref{sec:diffeo_reg} contains an introduction to
diffeomorphic curve matching and the associated Hamiltonian systems,
We also discuss the application of the finite element approach
using the Wu-Xu element from \cite{wu2019nonconforming} and the discretisation
of the velocity equation. Section \ref{sec:pullback_theory} contains
the transformation theory for the Wu-Xu element, and Section
\ref{sec:discretisation} contains details of the discretisation of the
Hamiltonian equations. Next, Section \ref{sec:inverse_problem} discusses the
Bayesian inverse problem, and Section \ref{sec:applications} contains
numerical results. Section \ref{sec:conclusion} contains a summary.

%% file: diffeo.tex
\section{Diffeomorphic registration}\label{sec:diffeo_reg}

Let $\Omega$ be a 
connected convex subset of $\Rd$, $d=2$, with polygonal boundary $\partial\Omega$.
We study maps $q\in Q=H^1(S^1,\Rd)$ from a template curve $\Gamma_0\in\embomega$ to
a target curve $\Gamma_1\in\embomega$ whose motion is restricted by the differential
equation:
\begin{equation}\label{eq:qmotion}
\dot{q}_t = u_t\circ q_t\,,
\end{equation}
where $u_t$, $t \in [0,1]$ is a family of time-dependent vector fields on
$\Omega$ with some prescribed spatial smoothness. A \emph{geodesic path} between
two such parameterised curves $\Gamma_0$ and $\Gamma_1$ is defined as a path
minimising the associated kinetic energy in $u$:
\begin{equation}\label{eq:kinetic_energy}
\frac{1}{2}\int_0^1\|u_t\|^2 \diff t,
\end{equation}
where $\|\cdot\|$ dominates the Lipschitz norm.
In fact, since $u_t$ is supported on $\Omega$ it generates a curve on
$\diffomega$ \cite{younes2010shapes} of the entire ambient space via:
\begin{equation}\label{eq:varphi}
\dot{\varphi}_t = u_t\circ \varphi_t,\quad\varphi_0 = \idop,
\end{equation}
whose motion restricted to the curve $q_0 \circ S^1$ equals the $q_t \circ S^1$
at time $t \in [0,1]$. As the kinetic energy measures distances between two
elements of $\emb$ via velocity defined over the entire field $\Omega$, we refer
to this associated distance measure as an \emph{outer} metric on the shape space
$\emb$.

\subsection{Hamiltonian system}\label{sec:hamiltonian}

Here we take a Hamiltonian approach \cite{glaunes2006modeling} and introduce
the momentum $p_t\in T^*Q$ occupying the linear cotangent space, which we assume has enough 
regularity so that it has a Fr\'echet-Riesz representer in $L^2(S^1)$ (also
denoted $p_t$, with some abuse of notation). We extremise the following the
functional:
\begin{align}\nonumber
S = \int_0^1 \frac{1}{2}\|u_t\|^2 + \langle p_t, \dot{q}_t- u_t\circ q_t\rangle\diff t,
\end{align}
where $\langle h, g \rangle= \int_{S^1} h\cdot g \dth$. Taking variations i.e.
$\delta S = 0$ leads to Hamilton's equations for curve matching for
$t\in [0,1]$:
\begin{subequations}\label{eq:hamiltonseqs}
\begin{align}
&  \int_0^1\langle\delta p, \dot{q}_t - u_t\circ q_t\rangle \diff t = 0,
&  \forall \delta p\in L^2(S^1), \label{eq:hamiltonseqs:q_update}\\
&  \int_0^1\langle\dot{p}_t - \nabla u_t\transp\circ q_t p_t, \delta q\rangle \diff t = 0,
& \forall \delta q \in Q, \label{eq:hamiltonseqs:p_update}\\
&  \half \frac{\delta \| u_t\|^2}{\delta u}  - \langle p_t, \delta u\circ q_t\rangle = 0.\label{eq:hamiltonseqs:u_update} &
\end{align}
\end{subequations}
where $\delta p$, $\delta u$ and $\delta q$ are space-time test functions.
The following theorem shows that we can solve \eqref{eq:hamiltonseqs:p_update}
analytically:
\begin{theorem}\label{thm:momchar}
The solution $p_t$ to \eqref{eq:hamiltonseqs:p_update} is at all times $t\geq 0$ given by $p_t = \nabla
\varphi_t\invtransp\circ q_0 p_0$.
\end{theorem}
\begin{proof}
See~\ref{app:momchar}.    
\end{proof}

To generate \emph{parameterisation-independent} geodesics as in
\cite{cotter2009geodesic} we replace the initial condition $q_0$ by
$q_0\circ\eta$, where $\eta\in\text{Diff}_+(S^1)$ in the case of planar curves
is an arbitrary reparameterisation. As a result of this quotient representation
$\emb\setminus\text{Diff}_+(S^1)$ of curves we minimise over all $\eta$ leading to the
\emph{horizontality condition} on the momentum. This means that the
momentum $p_0$ has no tangential component and can therefore
be described by a one-dimensional signal, $\tilde{p}_0: S^1\mapsto\mathbb{R}$:
\begin{align}\nonumber
p_0 = \mathbf{n}_{q_0} \tilde{p}_0
\end{align}
where $\mathbf{n}_{q_0}:S^1\rightarrow\mathbb{R}^2$ is the outward normal
of the template. Thus, along with Theorem \ref{thm:momchar} we have the 
following characterisation,
\begin{equation}\label{eq:reduced_momentum}
p_t = \varphi_t\invtransp\circ q_0\invtransp \mathbf{n}_{q_0} \tilde{p}_0.
\end{equation}
This generates trajectories of geodesics between unparameterised 
curves. The entire geodesic motion of the curve can therefore be determined 
by a  one-dimensional signal along the initial curve $q_0$.
To summarise this section we are concerned with integration of the 
following reduced Hamiltonian system for $t\in [0,1]$:
\begin{subequations}\label{eq:hamiltonseqs_reduced}
\begin{align}
&  \half \frac{\delta \| u_t\|^2}{\delta u}  = \langle\varphi_t\invtransp\circ q_0\mathbf{n}_{q_0} \tilde{p}_0, \delta u\circ q_t\rangle,
& \label{eq:hamiltonseqs_reduced:u}\\
&  \dot{q}_t = u_t\circ q_t,\label{eq:hamiltonseqs_reduced:q}
\end{align}
\end{subequations}
with $q_0$ and $\tilde{p}_0$ fixed and boundary conditions $u_t|_{\partial\Omega}=0$
for all $t\in [0,1]$. Next we discuss a discretisation of \eqref{eq:hamiltonseqs_reduced}.

\subsection{Outer metric via finite elements}\label{sec:outer_metric}

From Picard-Lindelh\"of analysis it is clear that the Banach space 
ordinary differential equation (ODE) \eqref{eq:hamiltonseqs_reduced:q} require 
a pointwise Lipschitz condition on $u_t$. As such, $u_t$ must occupy at least 
$\lipsobo$ when $q_0 \in L^\infty(S^1)$, see \cite[Theorem
5]{cotter2013bayesian} (see also Corollary 7 in this reference for other host
spaces). Dupuis \cite{dupuis1998variational} establishes
sufficient conditions accomplishing the same in a Hilbertian setting. The
Hilbertian setting is better suited to finite element methods. 
This is in contrast with $\lipsobo$ which is only a Banach space and, to the
best of the authors' ability, is not easy to approximate
numerically\footnote{\cite{lakkis2015adaptive} approximates by means of a fixed
point linearisation solutions to the nonlinear $\infty$-harmonic equation
\cite{barron2008infinity}.}. We therefore request a norm $\|\cdot\|$
such a way that a solution to \eqref{eq:hamiltonseqs_reduced:u} ensures that 
this condition is met, which in turn implies global existence and uniqueness of
\eqref{eq:hamiltonseqs_reduced:q} by the references above. For $d=2,\,3$,
$H_0^3(\Omega)$ is contained in
$\mathsf{C}^1(\bar{\Omega})$ and so is Lipschitz on the interior \cite[Theorem
2.5.1]{ziemer2012weakly}. As such, we want to describe a discretisation of
\eqref{eq:hamiltonseqs_reduced:u} ensuring a type of $H^3$ regularity as the
follow theorem shows.

\begin{theorem}\label{thm:H3Lip}
Let $O$ be a convex bounded Lipschitz domain in $\mathbb R^d$ with polygonal
boundary and $O_h$ a shape-regular, quasi-uniform triangulation thereof
\cite{ErnGuermond2013} for some mesh size $h>0$. Suppose further that 
$u$ is continuous on $\bar{O}$, $u|_K \in H^3(K)^d$ for $K\in O_h$ and that there exists an operator $B$
inducing the norm $\|u\|_B^2 = \sum_{K\in O_h} \|u\|_{B(K)}^2$, where we define
$\|u\|_{B(K)}^2 = \int_K Bu\cdot u\diff x$ such that $\|u\|_{H^3(K)^d}\lesssim \|u\|_{B(K)}$. Then $u\in W^{1,\infty}(O)^d$.
\end{theorem}
\begin{proof}
The embedding theorem for homogeneous Sobolev spaces (i.e. with zero traces)
into the space $\mathsf{C}^j(\bar{O})$ are well-known. However, since the trace
$\gamma_K u$ of $u$ on $\partial K$, $K\in O_h$ may not be zero. By \cite[Theorem 4.12]{adams2003sobolev}, $H^3(K) \hookrightarrow \mathsf{C}_B^1(K)$, where:
\begin{align}\nonumber
\mathsf{C}_B^1(K) = \{ u \in \mathsf{C}^1(K)\;|\; D^\bfalpha u\textnormal{ is
bounded on } K,\; |\bfalpha|\leq 1\}. 
\end{align}
This means any $H^3(K)$ function has a continuous representative with almost
everywhere bounded first derivatives on $K$. Since
$u\in\mathsf{C}^0(\bar{O})$, $u$ is a continuous function with its first
derivative a.e. bounded, implying a Lipschitz condition. To summarise:
\begin{align*}
\| u \|_{W^{1,\infty}(K)^d}^2 & \lesssim \| u \|_{H^3(K)^d}^2 \lesssim \| u \|_{B(K)}^2
\end{align*}
Summing over the elements $K\in\Omega$ and squaring:
\begin{align*}
\| u \|_{W^{1,\infty}(O)^d}^2\lesssim \| u \|_B^2.
\end{align*}
where we have used that $u$ is a continuous function with essentially bounded
gradient.
\end{proof}
In light of this theorem we approximate the space of velocity fields by a 
nonconforming finite element space (see e.g. \cite[Section 10.3]{brenner2008mathematical})
This way we can guarantee the necessary Lipschitz properties of our functions
without having to impose higher-order \emph{global} continuity of the finite-dimensional solution spaces.\\

In Section \ref{sec:discretisation} we use the $H^3$-nonconforming finite element
space presented in \cite[Section 4]{wu2019nonconforming} in a discretisation
of \eqref{eq:hamiltonseqs_reduced}. We choose the operator $B=(\idop - \alpha\Delta)^{2m}$
for a given positive constant $\alpha$ leading to the following bilinear form:
\begin{equation}\label{eq:bilinear_form}
a_\Omega(u, v) = \sum_{i=1}^d\int_{\Omega} \sum_{j=0}^m \alpha^j \binom{m}{j} D^j u^i\cdot
D^j v^i \diff x = \int_\Omega Bu\cdot v\diff x,
\end{equation}
where $x\cdot y$ is the Euclidean inner product, $D^0 = \idop$, and
\[
D^j =\begin{cases} \nabla D^{j-1} & j\text{ is odd},\\
\nabla\cdot D^{j-1} & j\text{ is even}. \end{cases}
\]

\section{A pullback theory for the Wu-Xu element}\label{sec:pullback_theory}

The Wu-Xu element provides an opportunity to tackle this problem in a (nonconforming) $H^3$ setting, but it presents challenges for implementation.  Although we can construct its basis on a reference element, say, using the FIAT package~\cite{Kirby:2004}, the Wu-Xu elements do not form an affine equivalent family~\cite{brenner2008mathematical} under pullback.
Consequently, we apply the theory developed in~\cite{kirby2018general}, which gives a generalization of techniques developed for the $C^1$ conforming Argyris element~\cite{argyris1968tuba,dominguez2008algorithm}.\\

To fix ideas, put a reference triangle $\widehat{K}$ with vertices by $\{ \widehat{\bfv}_i \}_{i=1}^3$.
For any nondegenerate triangle $K$ with vertices $\{ \bfv_i \}_{i=1}^3$, we
let $F:T \rightarrow \widehat{K}$ denote the affine mapping sending
each $\bfv_i$ to the corresponding $\widehat{\bfv}_i$ and $J_T$ its
Jacobian matrix.

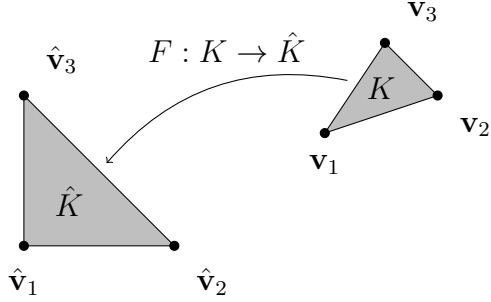
\begin{figure}[h!]
  \begin{center}
  \begin{tikzpicture}
    \draw[fill=lightgray] (0,0) coordinate (vhat1)
    -- (2,0) coordinate(vhat2)
    -- (0,2) coordinate(vhat3)--cycle;
    \foreach \pt\labpos\lab in {vhat1/below/\hat{\mathbf{v}}_1, vhat2/below right/\hat{\mathbf{v}}_2, vhat3/above right/\hat{\mathbf{v}}_3}{
      \filldraw (\pt) circle(.6mm) node[\labpos=1.5mm, fill=white]{$\lab$};
    }
    \draw[fill=lightgray] (4.0, 1.5) coordinate (v1)
    -- (5.5, 2.0) coordinate (v2)
    -- (4.8, 2.7) coordinate (v3) -- cycle;
    \foreach \pt\labpos\lab in {v1/below/\mathbf{v}_1, v2/below right/\mathbf{v}_2, v3/above right/\mathbf{v}_3}{
      \filldraw (\pt) circle(.6mm) node[\labpos=1.5mm, fill=white]{$\lab$};
    }
    \draw[<-] (1.1, 1.1) to[bend left] (4.3, 2.2);
    \node at (2.7, 2.65) {$F:K\rightarrow\hat{K}$};
    \node at (0.6,0.6) {$\hat{K}$};
    \node at (4.75, 2.1) {$K$};
  \end{tikzpicture}
  \end{center}
  \caption{Affine mapping to a reference cell \(\hat{K}\) from a
    typical cell \( K \).  Note that here $F$ maps from the physical
    cell $K$ to the reference cell $\hat{K}$ rather than the other way
  around.}
    \label{fig:affmap}
\end{figure}

We adopt the ordering convention used in \cite{rognes2010efficient}, where edge
$e_i$ of any triangle connects the vertices other than
$i$.  We take the unit tangent $\bft_i = \begin{bmatrix} t_i^x
  & t_i^y \end{bmatrix}^T$ to from the vertex of lower number to the
higher one.  The normal to edge $i$ is defined by counterclockwise
rotation of the tangent, so that $\bfn_i = R \bft_i$, where
$R = \left[ \begin{smallmatrix} 0 & 1 \\ -1 & 0 \end{smallmatrix} \right]$.
The normals,
tangents, and edge midpoints for the reference element $\widehat{K}$
will include hats: $\widehat{\bfn}_i$, $\widehat{\bft}_i$, and
$\widehat{\bfe}_i$. The pull-back of any function $\widehat{f}$ defined 
on $\widehat{K}$ is given by
\begin{equation}
  F^*(\widehat{f}) = \widehat{f} \circ F,
\end{equation}
and the push-forward of functionals $n$ acting on functions defined over
$K$ is
\begin{equation}
  F_*(n) = n \circ F^*,
\end{equation}
so that
\begin{equation}
  F_*(n)(\widehat{f}) = (n \circ F^*)(\widehat{f}) = n(\widehat{f}
  \circ F)
\end{equation}

Finite element implementation requires local shape functions $\{ \psi_i^K \}_{i=1}^N$  that are restrictions of the global basis to cell $K$.
These are taken dual to a set of \emph{nodes} or \emph{degrees of freedom} $\{ n_i^K \}_{i=1}^N$ in the sense that
\[
n_i^K(\psi_j^K) = \delta_{ij}.
\]

In practice, one typically computes the basis $\{ \hat{\psi}_i \}_{i=1}^N$ dual to some nodes $\{\hat{n}_i\}_{i=1}^N$ over the reference element $\hat{K}$. For affine equivalent families (like the Lagrange basis), the physical basis functions are the pullbacks of reference element shape functions, so that
\[
\psi_i^K = F^*(\widehat{\psi}_i).
\]
Equivalently, the nodes are preserved under push-forward, with
\[
F_*(n^K_i) = \hat{n}_i.
\]

We may express these relations in a kind of vector-notation.  
If $\hat{\Psi}$ is a vector whose entries are $\widehat{\psi}_i$, then in the affine equivalent case, $F^*(\hat{\Psi})$ contains the basis on cell $K$, and also $F_*(\mathcal{N}) = \widehat{\mathcal{N}}$.
For non-equivalent families, these relations fail, but we can hope to construct a matrix $M$ such that
\begin{equation}
  \Psi = M F^*(\hat{\Psi})
\end{equation}
contains the correct vector of basis functions on $T$.
The matrix $M$ will depend on the particular geometry of each cell, but if it is sparse this amounts to a considerable savings over directly constructing the basis on each triangle.
Our theory in~\cite{kirby2018general} proceeds by transforming the actions of the functionals on the finite element space.
The finite element functionals are defined on some infinite-dimensional space (e.g. twice-continuously differentiable functions), and we let $\pi$ denote the restriction of functionals to the finite-element space and $\hat{\pi}$ the corresponding restriction on the reference element.  Then, we look for a matrix $V$ such that
\begin{equation}
  \label{eq:V}
V F_*(\pi \mathcal{N}) = \widehat{\pi}\widehat{\mathcal{N}},
\end{equation}
and can prove~\cite[Theorem~3.1]{kirby2018general} that
\begin{equation}
  M = V^T.
\end{equation}


For any triangle $K$ and integer $k\geq 0$, we let $\mcp^k(K)$ denote the
space of polynomials of degree no greater than $k$ over $K$.  Letting
$\lambda_i$ be the barycentric coordinates for $K$ (equivalently, the
Lagrange basis for $\mcp^1(K)$), we let
$b_K = \lambda_1 \lambda_2 \lambda_3$ be the standard cubic bubble
function over $K$. We also need notation for the linear functionals defining 
degrees of freedom.  We let $\delta_{\bfx}$ denote pointwise evaluation of some
(continuous) function:
\begin{equation}
  \delta_\bfx(p) = p(\bfx).
\end{equation}
We let $\delta_\bfx^\bfx$ denote the derivative in some
direction $\bfs$ at a point $\bfx$:
\begin{equation}
    \delta^\bfs_\bfx(p) = \bfs^T \nabla p(\bfx)
\end{equation}
Repeated superscripts will indicate higher derivatives.
We use block notation will for gradients and sets of
second-order derivatives, such as
\begin{equation}
  \nabla_\bfx = \begin{bmatrix} \delta_\bfx^\bfx &
    \delta_{\bfx}^\bfy \end{bmatrix}^T
\end{equation}
for the gradient in Cartesian coordinates at a point $\bfx$, and
\begin{equation}
  \bigtriangleup_\bfx =
  \begin{bmatrix} \delta_{\bfx}^{\bfx\bfx} & \delta_{\bfx}^{\bfx\bfy} & \delta_{\bfx}^{\bfy\bfy}
  \end{bmatrix}^T
\end{equation}
for the unique components of the Hessian matrix.  We will use
superscripts in the block notation to indicate the derivatives taken
in other directions than the Cartesian ones, such as
$\nabla^{\bfn\bft}$ containing the derivatives with respect to a
normal vector $\bfn$ and tangent vector $\bft$ for some part of the
boundary.  Similarly, $\bigtriangleup^{\bfn\bft}$ will contain the
second partials in each direction and the mixed partial in both directions.\\

The Wu-Xu elements also utilise integral moments of normal derivatives, and we shall also need averages tangential and mixed derivatives over edges to perform the transformations. Given any directional vector $\bfs$, we define the moment of the derivative in the direction $\bfs$ over edge $\bfe$ by:

\begin{equation}
  \mu^\bfs_\bfe(f) = \int_\bfe \bfs \cdot \nabla f \, ds, \\
\end{equation}

Similarly, we let $\mu^{\bfs_1\bfs_2}_\bfe$ to denote the  functionals computing moments of second (possibly mixed) directional derivatives over an edge. Now, we define the pair of $H^3$ nonconforming triangles considered in~\cite{wu2019nonconforming}. 
Note that there are two spaces given: a space compatible with sixth-order problems,
and a \emph{robust} space that is stable for second, fourth and sixth-order problems.
We define function space $\wx(K)$ over some triangle $K$ by
\begin{equation}\label{eq:wxK}
  \wx(K) = \mcp^3 + b_K \mcp^1,
\end{equation}
and the function space for the robust element will be
\begin{equation}
  \wxr(K) = \mcp^3 + b_K \mcp^1 + b_K^2 \mcp^1,
\end{equation}
where $\mcp^k$ is the standard space of polynomials of degree $k$.
Note that we have $\dim \wx(K)= 12$ and $\dim \wxr(K) = 15$ since $b_K \in \mcp^3 \cap b_K \mcp^1$.  
The degrees of freedom for the two elements are quite similar.
We can parametrise $\wxr(K)$ by
\begin{equation}
  \label{eq:nodes}
  \mathcal{N} =
  \begin{bmatrix}
    \delta_{\bfv_1} & \nabla^T_{\bfv_1} &
    \delta_{\bfv_2} & \nabla^T_{\bfv_2} &
    \delta_{\bfv_3} & \nabla^T_{\bfv_3} &
    \mu^{\bfn_1\bfn_1}_{\bfe_1} & \mu^{\bfn_2\bfn_2}_{\bfe_2} &
    \mu^{\bfn_3\bfn_3}_{\bfe_3}
  \end{bmatrix}^T.
\end{equation}
That is, the degrees of freedom consist of point values and gradients at each vertex, together with moments of the second normal derivative along edges. For the robust element, we also use the moments of the first normal derivatives, so that
\begin{equation}
  \label{eq:robustnodes}
  \mathcal{N}_r =
  \begin{bmatrix}
    \delta_{\bfv_1} & \nabla^T_{\bfv_1} &
    \delta_{\bfv_2} & \nabla^T_{\bfv_2} &
    \delta_{\bfv_3} & \nabla^T_{\bfv_3} &
    \mu^{\bfn_1}_{\bfe_1} & \mu^{\bfn_2}_{\bfe_2} &
    \mu^{\bfn_3}_{\bfe_3} &
    \mu^{\bfn_1\bfn_1}_{\bfe_1} & \mu^{\bfn_2\bfn_2}_{\bfe_2} &
    \mu^{\bfn_3\bfn_3}_{\bfe_3}
  \end{bmatrix}^T.
\end{equation}

Wu and Xu actually define the degrees of freedom as average of these moments over the relevant facets, although this does not affect unisolvence or other essential properties.
For the reference element, it will be helpful to use their original definition.
For some edge $\bfe$ of $\hat{K}$, define
\begin{equation}
  \hat{\mu}^{\hat{\bfs}}_{\hat{\bfe}}(f) = \tfrac{1}{|\hat{\bfe}|}\int_{\hat{\bfe}} \hat{\bfs} \cdot \hat{\nabla} f \, d\hat{s},
\end{equation}
and similarly define moments second directional derivatives over reference element edges.
The reference element nodes for $\wx(\hat{K})$ will be taken as
\begin{equation}
  \label{eq:refnodes}
  \widehat{\mathcal{N}} =
  \begin{bmatrix}
    \delta_{\hat{\bfv}_1} & \hat{\nabla}^T_{\hat{\bfv}_1} &
    \delta_{\hat{\bfv}_2} & \hat{\nabla}^T_{\hat{\bfv}_2} &
    \delta_{\hat{\bfv}_3} & \hat{\nabla}^T_{\hat{\bfv}_3} &
    \hat{\mu}^{\hat{\bfn}_1\hat{\bfn}_1}_{\hat{\bfe}_1} &
    \hat{\mu}^{\hat{\bfn}_2\hat{\bfn}_2}_{\hat{\bfe}_2} &
    \hat{\mu}^{\hat{\bfn}_3\hat{\bfn}_3}_{\hat{\bfe}_3}
  \end{bmatrix}^T,
\end{equation}
and for $\wxr(\hat{K})$ we will use
\begin{equation}
  \label{eq:refnodesrobust}
  \widehat{\mathcal{N}} =
  \begin{bmatrix}
    \delta_{\hat{\bfv}_1} & \hat{\nabla}^T_{\hat{\bfv}_1} &
    \delta_{\hat{\bfv}_2} & \hat{\nabla}^T_{\hat{\bfv}_2} &
    \delta_{\hat{\bfv}_3} & \hat{\nabla}^T_{\hat{\bfv}_3} &
    \hat{\mu}^{\hat{\bfn}_1}_{\hat{\bfe}_1} &
    \hat{\mu}^{\hat{\bfn}_2}_{\hat{\bfe}_2} &
    \hat{\mu}^{\hat{\bfn}_3}_{\hat{\bfe}_3} &
    \hat{\mu}^{\hat{\bfn}_1\hat{\bfn}_1}_{\hat{\bfe}_1} &
    \hat{\mu}^{\hat{\bfn}_2\hat{\bfn}_2}_{\hat{\bfe}_2} &
    \hat{\mu}^{\hat{\bfn}_3\hat{\bfn}_3}_{\hat{\bfe}_3}
  \end{bmatrix}^T
\end{equation}
Note that this redefinition has no effect in the case of an equilateral reference triangle with unit edge length.
For the more common case of a right isosceles reference triangle, however, this will eliminate the need for logic indicating to which reference element edges the edges of each triangle correspond.\\

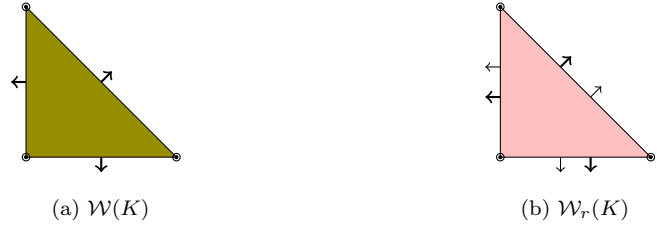
\begin{figure}
  \centering
    \begin{subfigure}[t]{0.45\textwidth}    
      \centering
      \begin{tikzpicture}[scale=1.0]
        \draw[fill=olive] (0,0) -- (2, 0) -- (0, 2) -- cycle;
        \foreach \i/\j in {0/0, 2/0, 0/2}{
          \draw[fill=black] (\i, \j) circle (0.02);
          \draw (\i, \j) circle (0.05);
        }
        \foreach \i/\j/\n/\t in {1/0.0/0.0/-1, 1/1/0.707/0.707, 0.0/01/-1/0}{
          \draw[thick,-{>[scale=0.5]}] (\i, \j) -- (\i+\n/5, \j+\t/5);
        }
      \end{tikzpicture}
      \label{regwx}
      \caption{$\wx(K)$}
    \end{subfigure}
    \begin{subfigure}[t]{0.45\textwidth}
      \centering
      \begin{tikzpicture}[scale=1.0]
        \draw[fill=pink] (0,0) -- (2, 0) -- (0, 2) -- cycle;
        \foreach \i/\j in {0/0, 2/0, 0/2}{
          \draw[fill=black] (\i, \j) circle (0.02);
          \draw (\i, \j) circle (0.05);
        }
        \foreach \i/\j/\n/\t in {0.8/0.0/0.0/-1, 1.2/0.8/0.707/0.707, 0.0/1.2/-1/0}{
          \draw[->] (\i, \j) -- (\i+\n/5, \j+\t/5);
        }
        \foreach \i/\j/\n/\t in {1.2/0.0/0.0/-1, 0.8/1.2/0.707/0.707, 0.0/0.8/-1/0}{
          \draw[thick,-{>[scale=0.5]}] (\i, \j) -- (\i+\n/5, \j+\t/5);
        }    
      \end{tikzpicture}
      \label{robustwx}
      \caption{$\wxr(K)$}      
    \end{subfigure}
    \caption{Degrees of freedom for the Wu-Xu (left) and robust Wu-Xu (right) elements.
      Point values are given by dots, gradients by cirles, while averages of first and second  normal derivatives are given by thin and thick arrows, respectively.}
  \label{fig:wuxudofs}
\end{figure}

The derivative degrees of freedom in both Wu-Xu elements are not preserved under push-forward, and since we have only normal derivatives on the edges, we cannot immediately obtain the correct nodes by taking linear combinations.
Consequently, we must develop a \emph{compatible nodal completion}~\cite[Definition~3.4]{kirby2018general}.
For the Wu-Xu elements, this contains all the original degrees of freedom plus the integrals of tangential and mixed normal/tangential derivatives.
Such a completion is shown for the standard Wu-Xu element in Figure~\ref{wxcomp}. A completion for the robust element includes the first normal moments and tangential moments as well, as showin in Figure~\ref{rwxcomp}.\\

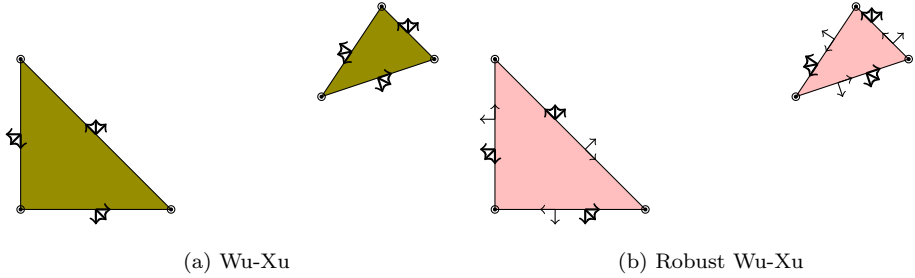
\begin{figure}[h!]
  \centering
  \begin{subfigure}[t]{0.45\textwidth}
    \begin{tikzpicture}[scale=1.0]
      \draw[fill=olive] (0,0) -- (2, 0) -- (0, 2) -- cycle;
      \foreach \i/\j in {0/0, 2/0, 0/2}{
        \draw[fill=black] (\i, \j) circle (0.02);
        \draw (\i, \j) circle (0.05);
      }
      \foreach \i/\j/\n/\t in {1/0.0/0.0/-1, 1/1/0.707/0.707, 0/1/-1/0}{
        \draw[thick,-{>[scale=0.5]}] (\i, \j) -- (\i+\n/5, \j+\t/5);
        \draw[thick,-{>[scale=0.5]}] (\i, \j) -- (\i-\t/5, \j+\n/5);
        \draw[thick,-{>[scale=0.5]}] (\i, \j) -- (\i-0.1414*\t+.1414*\n, \j+.1414*\t+.1414*\n);      
      }
      
      \draw[fill=olive] (4.0, 1.5) -- (5.5, 2.0) -- (4.8, 2.7) -- cycle;
      \foreach \x\y in {4/1.5, 5.5/2, 4.8/2.7}{
        \filldraw (\x,\y) circle(.02);
        \draw (\x,\y) circle(.05);
      }
      \foreach \x\y\nx\ny\tx\ty in {5.15/2.35/.707/.707/-.707/.707,4.4/2.1/-0.8321/0.5547/-.5547/-.8321,4.75/1.75/.3163/-.9487/.9487/.3163}{
        \draw[thick,-{>[scale=0.5]}] (\x,\y) -- (\x+0.2*\nx,\y+0.2*\ny);
        \draw[thick,-{>[scale=0.5]}] (\x,\y) -- (\x+0.2*\tx,\y+0.2*\ty);
        \draw[thick,-{>[scale=0.5]}] (\x,\y) --
        (\x+0.1414*\tx+0.1414*\nx, \y+0.1414*\ty+0.1414*\ny);
      }
    \end{tikzpicture}
    \caption{Wu-Xu}
    \label{wxcomp}
  \end{subfigure}
    \begin{subfigure}[t]{0.45\textwidth}
      \begin{tikzpicture}[scale=1.0]
        \draw[fill=pink] (0,0) -- (2, 0) -- (0, 2) -- cycle;
        \foreach \i/\j in {0/0, 2/0, 0/2}{
          \draw[fill=black] (\i, \j) circle (0.02);
          \draw (\i, \j) circle (0.05);
        }
        \foreach \i/\j/\n/\t in {0.8/0.0/0.0/-1, 1.2/0.8/0.707/0.707, 0.0/1.2/-1/0}{
          \draw[->] (\i, \j) -- (\i+\n/5, \j+\t/5);
          \draw[->] (\i, \j) -- (\i+\t/5, \j-\n/5);
        }
        \foreach \i/\j/\n/\t in {1.2/0.0/0.0/-1, 0.8/1.2/0.707/0.707, 0.0/0.8/-1/0}{
          \draw[thick,-{>[scale=0.5]}] (\i, \j) -- (\i+\n/5, \j+\t/5);
          \draw[thick,-{>[scale=0.5]}] (\i, \j) -- (\i-\t/5, \j+\n/5);
          \draw[thick,-{>[scale=0.5]}] (\i, \j) -- (\i-0.1414*\t+.1414*\n, \j+.1414*\t+.1414*\n);      
        }
  
        \draw[fill=pink] (4.0, 1.5) -- (5.5, 2.0) -- (4.8, 2.7) -- cycle;
        \foreach \x\y in {4/1.5, 5.5/2, 4.8/2.7}{
          \filldraw (\x,\y) circle(.02);
          \draw (\x,\y) circle(.05);
        }
        \foreach \x\y\nx\ny\tx\ty in {5.15/2.35/.707/.707/-.707/.707,4.4/2.1/-0.8321/0.5547/-.5547/-.8321,4.75/1.75/.3163/-.9487/.9487/.3163}{
          \draw[->] (\x-0.2*\tx,\y-0.2*\ty) -- (\x-0.2*\tx+0.2*\nx,\y-0.2*\ty+0.2*\ny);
          \draw[->] (\x-0.2*\tx,\y-0.2*\ty) -- (\x-0.2*\tx+0.2*\tx,\y-0.2*\ty+0.2*\ty);
        }
        \foreach \x\y\nx\ny\tx\ty in {5.15/2.35/.707/.707/-.707/.707,4.4/2.1/-0.8321/0.5547/-.5547/-.8321,4.75/1.75/.3163/-.9487/.9487/.3163}{
          \draw[thick,-{>[scale=0.5]}] (\x+0.2*\tx,\y+0.2*\ty) -- (\x+0.2*\tx+0.2*\nx,\y+0.2*\ty+0.2*\ny);
          \draw[thick,-{>[scale=0.5]}] (\x+0.2*\tx,\y+0.2*\ty) -- (\x+0.2*\tx+0.2*\tx,\y+0.2*\ty+0.2*\ty);
          \draw[thick,-{>[scale=0.5]}] (\x+0.2*\tx,\y+0.2*\ty) --
          (\x+0.2*\tx+0.1414*\tx+0.1414*\nx, \y+0.2*\ty+0.1414*\ty+0.1414*\ny);
        }    
      \end{tikzpicture}
      \caption{Robust Wu-Xu}
    \label{rwxcomp}      
    \end{subfigure}
    \caption{Compatible nodal completions for the Wu-Xu and robust Wu-Xu elements}
  \label{fig:wuxucompletion}
\end{figure}

We define
\begin{equation}
  \mathcal{M}_{1,i} =
  \begin{bmatrix}
    \mu^{\bfn_i}_{\bfe_i} &
    \mu^{\bft_i}_{\bfe_i} \end{bmatrix}^T
\end{equation}
to be the vector of the moments of the normal and tangential
derivatives on a particular edge.
We also let $\widehat{\mathcal{M}_{1, i}}$ contain the corresponding reference element nodes.
We only need $\mathcal{M}_{1,i}$ and $\widehat{\mathcal{M}_{1, i}}$ for the robust element.  Both elements require
\begin{equation}
  \mathcal{M}_{2,i} =
  \begin{bmatrix}
    \mu^{\bfn_i\bfn_i}_{\bfe_i} &
    \mu^{\bft_i\bft_i}_{\bfe_i} &
    \mu^{\bfn_i\bft_i}_{\bfe_i} \end{bmatrix}^T
\end{equation}
containing the unique second derivative moments on each edge.
We similarly define $\widehat{\mathcal{M}_{2, i}}$ to contain the reference element integral averages.
The compatible nodal completion for $(K, \mcp(K), \mathcal{N})$ is
\begin{equation}
  \label{eq:completenodes}
  \mathcal{N}^C =
  \begin{bmatrix}
    \delta_{\bfv_1} & \nabla^T_{\bfv_1} &
    \delta_{\bfv_2} & \nabla^T_{\bfv_2} &
    \delta_{\bfv_3} & \nabla^T_{\bfv_3} &
    \mathcal{M}_{2,1}^T &
    \mathcal{M}_{2,2}^T &
    \mathcal{M}_{2,3}^T
  \end{bmatrix}^T,
\end{equation}
with the hatted equivalents comprising $\hat{\mathcal{N}}^C$ on the reference cell.
The completed set of nodes for the robust element is
\begin{equation}
{
\scriptsize
  \label{eq:completerobustnodes}
  \mathcal{N}_r^C =
  \begin{bmatrix}
    \delta_{\bfv_1} & \nabla^T_{\bfv_1} &
    \delta_{\bfv_2} & \nabla^T_{\bfv_2} &
    \delta_{\bfv_3} & \nabla^T_{\bfv_3} &
    \mathcal{M}_{1,1}^T &
    \mathcal{M}_{1,2}^T &
    \mathcal{M}_{1,3}^T &
    \mathcal{M}_{2,1}^T &
    \mathcal{M}_{2,2}^T &
    \mathcal{M}_{2,3}^T
  \end{bmatrix}^T,
 }
\end{equation}

Now, the matrix $V$ from~\eqref{eq:V} will be obtained in factored form
\begin{equation}
  V = E V^c D,
\end{equation}
where each matrix plays a particular role.  $D$ is a rectangular matrix expressing the completed nodes in terms of the given physical nodes.  $V^c$ is a block diagonal matrix relating the push-forward of the reference nodal completion to the physical nodal completion, and $E$ is a Boolean matrix selecting actual finite element nodes from the completion.  For the Wu-Xu element, $D$ is $18 \times 12$, $V^c$ is $18 \times 18$, and $E$ is $12 \times 18$.  For the robust element, $D$ is $24 \times 15$, $V^c$ is $24 \times 24$, and $E$ is $15 \times 24$.\\

Now, we define the matrix $D$, which expresses the members of $\mathcal{N}^C$ as linear combinations of the members of $\mathcal{N}$.  
Clearly, the rows corresponding to members of $\mathcal{N}^C$ also appearing in $\mathcal{N}$
will just have a single nonzero in the appropriate column.  
For the Wu-Xu element, the remaining nodes are all integrals of quantities over 
edges, and we can use the Fundamental Theorem of Calculus to perform this task.
Let $\bfe$ be an edge running from vertex $\bfv_a$ to $\bfv_b$ with unit tangent and normal $\bft$ and $\bfn$,
respectively.  We have
\begin{equation}
\label{eq:mut}
  \mu^{\bft}_{\bfe}(f) = \int_{\bfe} \bft^T \nabla f ds
  = f(\bfv_b)-f(\bfv_a) = \delta_{\bfv_b}(f) - \delta_{\bfv_a}(f).
\end{equation}
In a similar way, the moments of the second tangential and mixed
derivatives on $\bfe$ can be expressed as differences between
components of the gradients at endpoints by:

\begin{equation}
\label{eq:mustuff}
  \begin{split}
    \mu^{\bfn\bft}_{\bfe}(f) & = \bfn^T \left(\nabla_{\bfv_b} f -
    \nabla_{\bfv_a} f\right), \\
    \mu^{\bft\bft}_{\bfe}(f) & = \bft^T \left(\nabla_{\bfv_b} f -
    \nabla_{\bfv_a} f\right),
  \end{split}
\end{equation}
and we have that $\mathcal{N}^C = D \mathcal{N}$, or
\begin{equation}
\label{eq:NCDNrobust}
{\tiny  \begin{bmatrix}
    \delta_{\bfv_1} \\
    \delta^\bfx_{\bfv_1} \\
    \delta^\bfy_{\bfv_1} \\
    \delta_{\bfv_2} \\
    \delta^\bfx_{\bfv_2} \\
    \delta^\bfy_{\bfv_2} \\
    \delta_{\bfv_3} \\ 
    \delta^\bfx_{\bfv_3} \\
    \delta^\bfy_{\bfv_3} \\ \hline
    \mu^{\bfn_1\bfn_1}_{\bfe_1} \\
    \mu^{\bfn_1\bft_1}_{\bfe_1} \\
    \mu^{\bft_1\bft_1}_{\bfe_1} \\
    \mu^{\bfn_2\bfn_2}_{\bfe_2} \\
    \mu^{\bfn_2\bft_2}_{\bfe_2} \\
    \mu^{\bft_2\bft_2}_{\bfe_2} \\
    \mu^{\bfn_3\bfn_3}_{\bfe_3} \\
    \mu^{\bfn_3\bft_3}_{\bfe_3} \\
    \mu^{\bft_3\bft_3}_{\bfe_3}
  \end{bmatrix}
  =
  \begin{bmatrix}
    1 & 0 & 0 & 0 & 0 & 0 & 0 & 0 & 0 & 0 & 0 & 0\\
    0 & 1 & 0 & 0 & 0 & 0 & 0 & 0 & 0 & 0 & 0 & 0\\
    0 & 0 & 1 & 0 & 0 & 0 & 0 & 0 & 0 & 0 & 0 & 0\\
    0 & 0 & 0 & 1 & 0 & 0 & 0 & 0 & 0 & 0 & 0 & 0\\
    0 & 0 & 0 & 0 & 1 & 0 & 0 & 0 & 0 & 0 & 0 & 0\\
    0 & 0 & 0 & 0 & 0 & 1 & 0 & 0 & 0 & 0 & 0 & 0\\ 
    0 & 0 & 0 & 0 & 0 & 0 & 1 & 0 & 0 & 0 & 0 & 0\\
    0 & 0 & 0 & 0 & 0 & 0 & 0 & 1 & 0 & 0 & 0 & 0\\
    0 & 0 & 0 & 0 & 0 & 0 & 0 & 0 & 1 & 0 & 0 & 0\\ \hline
    0 & 0 & 0 & 0 & 0 & 0 & 0 & 0 & 0 & 1 & 0 & 0\\
    0 & 0 & 0 & 0 & -n_{1,x} & -n_{1,y} & 0 & n_{1,x} & n_{1,y} &
    0 & 0 & 0 \\
    0 & 0 & 0 & 0 & -t_{1,x} & -t_{1,y} & 0 & t_{1,x} & t_{1,y} &
    0 & 0 & 0 \\
    0 & 0 & 0 & 0 & 0 & 0 & 0 & 0 & 0 & 0 & 1 & 0 \\
    0 & -n_{2,x} & -n_{2,y} & 0 & 0 & 0 & 0 & n_{2,x} & n_{2,y} &
    0 & 0 & 0 \\
    0 & -t_{2,x} & -t_{2,y} & 0 & 0 & 0 & 0 & t_{2,x} & t_{2,y} &
    0 & 0 & 0 \\    
    0 & 0 & 0 & 0 & 0 & 0 & 0 & 0 & 0 & 0 & 0 & 1 \\
    0 & -n_{3,x} & -n_{3,y} & 0 & n_{3,x} & n_{3,y} & 0 & 0 & 0 &
    0 & 0 & 0 \\
    0 & -t_{3,x} & -t_{3,y} & 0 & t_{3,x} & t_{3,y} & 0 & 0 & 0 &
    0 & 0 & 0 \\
  \end{bmatrix}
  \begin{bmatrix}
    \delta_{\bfv_1} \\
    \delta^\bfx_{\bfv_1} \\
    \delta^\bfy_{\bfv_1} \\
    \delta_{\bfv_2} \\
    \delta^\bfx_{\bfv_2} \\
    \delta^\bfy_{\bfv_2} \\
    \delta_{\bfv_3} \\
    \delta^\bfx_{\bfv_3} \\
    \delta^\bfy_{\bfv_3} \\ \hline
    \mu^{\bfn_1\bfn_1}_{\bfe_1} \\
    \mu^{\bfn_2\bfn_2}_{\bfe_2} \\
    \mu^{\bfn_3\bfn_3}_{\bfe_3} \\
  \end{bmatrix}.}
\end{equation}

The matrix $V^C$ is obtained by relating the push-forwards of the nodal completion to their reference counterparts.\\

We can convert between the Cartesian and other orthogonal coordinate
systems (e.g. normal/tangential) representations as follows.  Given a
pair of orthogonal unit vectors $\bfn$ and $\bft$, we can define an
orthogonal matrix $G$ by:
\begin{equation}
  G = \begin{bmatrix} \bfn & \bft \end{bmatrix}^T.
\end{equation}
In particular, we will use $G_i$ to have the normal and tangential
vectors to edge $i$ of triangle $K$ and $\widehat{G}_i$ those for
triangle $\widehat{K}$. The multivariate chain rule readily shows that
\begin{equation}
  \label{eq:gradtont}
  \nabla_x = G^T \nabla^{\bfn\bft}_\bfx.
\end{equation}
Similarly, letting $\bfn = \begin{bmatrix} n_x & n_y\end{bmatrix}^T$ and
$\bft = \begin{bmatrix} t_x & t_y \end{bmatrix}^T$, we define the
matrix $\Gamma$ by
\begin{equation}
  \Gamma = \begin{bmatrix}
    n_x^2 & 2 n_x t_x & t_x^2 \\
    n_x n_y & n_x t_y + n_y t_x & t_x t_y \\
    n_y^2 & 2 n_y t_y & t_y^2 \end{bmatrix},
\end{equation}
and the chain rule gives
\begin{equation}
  \bigtriangleup_x = \Gamma \bigtriangleup_x^{\bfn\bft}.
\end{equation}
Although $G$ is an orthogonal matrix, $\Gamma$ is not.
A similar calculation also shows gives that:
$\bigtriangleup^{\bfn\bft}_\bfx = \Gamma^{-1} \bigtriangleup_\bfx$,
where
\begin{equation}
  \Gamma^{-1} = \begin{bmatrix}
    n_x^2 & 2 n_x n_y & n_y^2 \\
    n_x t_x & n_x t_y + n_y t_x & n_y t_y \\
    t_x^2 & 2 t_x t_y & t_y^2 \end{bmatrix}.
\end{equation}

We will also need to transform derivatives under pull-back.  Using the
chain rule, 
\begin{equation}
\nabla (\hat{\psi} \circ F) = J^T \hat{\nabla} \hat{\psi} \circ F.
\end{equation}
Combining this with~\eqref{eq:gradtont} lets us relate the normal and
tangential derivatives in physical space to the normal and tangential
derivatives in reference space.
\begin{equation}
  \label{eq:transformnortangrad}
  \nabla_{\bfx}^{\bfn\bft} = G J^T \widehat{G}^T \hat{\nabla}^{\hat{\bfn}\hat{\bft}}_{\hat{\bfx}}.
\end{equation}

We can perform a similar calculation for second derivatives.  With
the entries of the Jacobian matrix as:
\begin{equation}
  J = \begin{bmatrix} \tfrac{\partial x}{\partial \hat{x}} &
    \tfrac{\partial x}{\partial \hat{y}} \\
    \tfrac{\partial y}{\partial \hat{x}} &
    \tfrac{\partial y}{\partial \hat{y}}
  \end{bmatrix},
\end{equation}
we define the matrix 
\begin{equation}
  \Theta
  = \begin{bmatrix}
    \left( \tfrac{\partial \hat{x}}{\partial x} \right)^2
    & 2 \tfrac{\partial \hat{x}}{\partial x} \tfrac{\partial
      \hat{y}}{\partial x}
    & \left( \tfrac{\partial \hat{y}}{\partial x} \right)^2 \\
    \tfrac{\partial \hat{x}}{\partial y}
    \tfrac{\partial \hat{x}}{\partial x}
    &
    \tfrac{\partial{\hat{x}}}{\partial y}
    \tfrac{\partial\hat{y}}{\partial x}
    + \tfrac{\partial \hat{x}}{\partial x}
    \tfrac{\partial \hat{y}}{\partial y}
    & \tfrac{\partial \hat{y}}{\partial x} \tfrac{\partial
      \hat{y}}{\partial y} \\
    \left(\tfrac{\partial \hat{x}}{\partial y} \right)^2
    & 2 \tfrac{\partial \hat{x}}{\partial y} \tfrac{\partial
      \hat{y}}{\partial y}
    & \left( \tfrac{\partial \hat{y}}{\partial y} \right)^2
    \end{bmatrix},
\end{equation}
so that for $\bfx = F(\hat{\bfx})$,
\begin{equation}
  \bigtriangleup_\bfx = \Theta \hat{\bigtriangleup}_{\hat{\bfx}}.
\end{equation}
The inverse of $\Theta$ follows by reversing
the roles of reference and physical variables:
\begin{equation}
  \Theta^{-1}
  = \begin{bmatrix}
    \left( \tfrac{\partial x}{\partial \hat{x}} \right)^2
    & 2 \tfrac{\partial x}{\partial \hat{x}}
    \tfrac{\partial y}{\partial \hat{x}}
    & \left( \tfrac{\partial y}{\partial \hat{x}} \right)^2 \\
    \tfrac{\partial x}{\partial \hat{y}}
    \tfrac{\partial x}{\partial \hat{x}}
    &
    \tfrac{\partial{x}}{\partial \hat{y}}
    \tfrac{\partial y}{\partial \hat{x}}
    + \tfrac{\partial x}{\partial \hat{x}}
    \tfrac{\partial y}{\partial \hat{y}}
    & \tfrac{\partial y}{\partial \hat{x}} \tfrac{\partial
      y}{\partial \hat{y}} \\
    \left(\tfrac{\partial x}{\partial \hat{y}} \right)^2
    & 2 \tfrac{\partial x}{\partial \hat{y}} \tfrac{\partial
      y}{\partial \hat{y}}
    & \left( \tfrac{\partial y}{\partial \hat{y}} \right)^2
    \end{bmatrix}
\end{equation}

We can also relate the second-order derivatives in normal/tangential
coordinates under pullback by
\begin{equation}
  \label{eq:pb2ndnt}
  \bigtriangleup^{\bfn\bft}_\bfx = \Gamma \Theta \hat{\Gamma}^{-1} \bigtriangleup_{\hat{\bfx}}^{\hat{\bfn}\hat{\bft}}.
\end{equation}
From here, we will let $G_i$ and $\hat{G}_i$ denote the matrices
containing normal and tangent vectors to edge $\bfe_i$ of a generic
triangle $T$ and the reference triangle $\hat{T}$, respectively, with
similar convention for the other geometric quantities
$\Gamma$ and $\Theta$. For any vector $\bfs$,
edge $\bfe$, and smooth function $f = f \circ F$, we have
\begin{equation}
\int_\bfe \bfs^T \nabla f ds =
\int_\bfe \bfs^T \hat\nabla f \circ F ds 
= \int_{\hat{\bfe}} \bfs^T \hat\nabla f J_{\bfe,\hat{\bfe}} d\hat{s},
\end{equation}
where the Jacobian $J_{\bfe,\hat{\bfe}}$ is just the ratio of the
length of $\bfe$ to that of the corresponding reference element edge
$\hat{\bfe}$. Applying this to the normal and tangential moments and
using~\eqref{eq:transformnortangrad}, we have that:
\begin{equation}
\mathcal{M}_{1,i}
= |\bfe_i| G_i J^T \hat{G}_{i}^{-1} \widehat{\mathcal{M}}_{1, i},
\end{equation}
where the factor of $|\hat{\bfe}_i|$ in the denominator of the
Jacobian is merged with the reference element moments to produce
$\widehat{\mathcal{M}_{1,i}}$.  Hence, the slight modification of
reference element nodes avoids extra data structures or logic in
identifying reference element edge numbers. Then, we can use~\eqref{eq:pb2ndnt} 
to express each $\mathcal{M}_{2, i}$ in terms of the reference element nodes
\begin{equation}
  \mathcal{M}_{2,i} = |\bfe| \Gamma_i \Theta \hat{\Gamma}_{i}^{-1}\widehat{\mathcal{M}}_{2,i}.
\end{equation}

We define vectors
\begin{equation}
  B^{1,i} = \tfrac{1}{|\bfe_i|} \hat{G}_i J^{-T} G_i^T, \ \ \
  B^{2,i} = \tfrac{1}{|\bfe_i|} \hat{\Gamma}_i \Theta^{-1} \Gamma_i^{-1},
\end{equation}
and hence $V^C$ is the block-diagonal matrix
\begin{equation}
  V^C = \begin{bmatrix}
    1 & & & & & & & & \\
    & J^{-T} & & & & & & & \\
    & & 1 & & & & & & \\
    & & & J^{-T} & & & & & \\
    & & & & 1 & & & & \\
    & & & & & J^{-T} & & \\
    & & & & & & B^{2,1} & & \\
    & & & & & & & B^{2,2} & \\
    & & & & & & & & B^{2,3}
    \end{bmatrix},
\end{equation}
with zeros of the appropriate shapes in the off-diagonal blocks.
The extraction matrix $E$ is just the $12 \times 18$ Boolean matrix selecting the members of $N$ from $N^C$:
\begin{equation}
E =
\begin{bmatrix}
1 & 0 & 0 & 0 & 0 & 0 & 0 & 0 & 0 & 0 & 0 & 0 & 0 & 0 & 0 & 0 & 0 & 0 \\
0 & 1 & 0 & 0 & 0 & 0 & 0 & 0 & 0 & 0 & 0 & 0 & 0 & 0 & 0 & 0 & 0 & 0 \\
0 & 0 & 1 & 0 & 0 & 0 & 0 & 0 & 0 & 0 & 0 & 0 & 0 & 0 & 0 & 0 & 0 & 0 \\
0 & 0 & 0 & 1 & 0 & 0 & 0 & 0 & 0 & 0 & 0 & 0 & 0 & 0 & 0 & 0 & 0 & 0 \\
0 & 0 & 0 & 0 & 1 & 0 & 0 & 0 & 0 & 0 & 0 & 0 & 0 & 0 & 0 & 0 & 0 & 0 \\
0 & 0 & 0 & 0 & 0 & 1 & 0 & 0 & 0 & 0 & 0 & 0 & 0 & 0 & 0 & 0 & 0 & 0 \\
0 & 0 & 0 & 0 & 0 & 0 & 1 & 0 & 0 & 0 & 0 & 0 & 0 & 0 & 0 & 0 & 0 & 0 \\
0 & 0 & 0 & 0 & 0 & 0 & 0 & 1 & 0 & 0 & 0 & 0 & 0 & 0 & 0 & 0 & 0 & 0 \\
0 & 0 & 0 & 0 & 0 & 0 & 0 & 0 & 1 & 0 & 0 & 0 & 0 & 0 & 0 & 0 & 0 & 0 \\
0 & 0 & 0 & 0 & 0 & 0 & 0 & 0 & 0 & 1 & 0 & 0 & 0 & 0 & 0 & 0 & 0 & 0 \\
0 & 0 & 0 & 0 & 0 & 0 & 0 & 0 & 0 & 0 & 0 & 0 & 1 & 0 & 0 & 0 & 0 & 0 \\
0 & 0 & 0 & 0 & 0 & 0 & 0 & 0 & 0 & 0 & 0 & 0 & 0 & 0 & 0 & 1 & 0 & 0 \\
\end{bmatrix}
\end{equation}

Multiplying $EV^CD$ out and defining
\begin{equation}
\label{eq:beta}
  \begin{split}
    \beta^{i,x} & = n^i_x B^{2,i}_{12} + t^i_x B^{2,i}_{13}, \\
    \beta^{i,y} & = n^i_y B^{2,i}_{12} + t^i_y B^{2,i}_{13},
  \end{split}
\end{equation}
we obtain for $V$
\begin{equation}
{\tiny
V =
\begin{bmatrix}
1 & 0 & 0 & 0 & 0 & 0 & 0 & 0 & 0 & 0 & 0 & 0\\
0 & \tfrac{\partial x}{\partial\hat{x}} & \tfrac{\partial y}{\partial \hat{x}} & 0 & 0 & 0 & 0 & 0 & 0 & 0 & 0 & 0\\
0 & \tfrac{\partial x}{\partial \hat{y}} & \tfrac{\partial y}{\partial \hat{y}} & 0 & 0 & 0 & 0 & 0 & 0 & 0 & 0 & 0\\
0 & 0 & 0 & 1 & 0 & 0 & 0 & 0 & 0 & 0 & 0 & 0\\
0 & 0 & 0 & 0 & \tfrac{\partial x}{\partial\hat{x}} & \tfrac{\partial y}{\partial\hat{x}} & 0 & 0 & 0 & 0 & 0 & 0\\
0 & 0 & 0 & 0 & \tfrac{\partial x}{\partial\hat{y}} & \tfrac{\partial y}{\partial\hat{y}} & 0 & 0 & 0 & 0 & 0 & 0\\
0 & 0 & 0 & 0 & 0 & 0 & 1 & 0 & 0 & 0 & 0 & 0\\
0 & 0 & 0 & 0 & 0 & 0 & 0 & \tfrac{\partial x}{\partial\hat{x}} & \tfrac{\partial y}{\partial\hat{x}} & 0 & 0 & 0\\
0 & 0 & 0 & 0 & 0 & 0 & 0 & \tfrac{\partial x}{\partial\hat{y}} & \tfrac{\partial y}{\partial\hat{y}} & 0 & 0 & 0\\
0 & 0 & 0 & 0 & -\beta^{1,x} & -\beta^{1,y} & 0 & \beta^{1, x}  & \beta^{1,y} & B^{2,1}_{11} & 0 & 0\\
0 & -\beta^{2, x} & -\beta^{2,y} & 0 & 0 & 0 & 0 & \beta^{2,x} & \beta^{2,y} & 0 & B^{2,2}_{11} & 0\\
0 & -\beta^{3,x} & -\beta^{3,y} & 0 & \beta^{3,x} & \beta^{3,y} & 0 & 0 & 0 & 0 & 0 & B^{2,3}_{11}\\
\end{bmatrix}.}
\end{equation}

The same considerations lead to a similar derivation of $E$, $V^c$, and $D$ for the robust element, resulting in 
{\setlength\arraycolsep{3pt}

\begin{equation}
  {\tiny
V = \begin{bmatrix}
1 & 0 & 0 & 0 & 0 & 0 & 0 & 0 & 0 & 0 & 0 & 0 & 0 & 0 & 0\\
0 & \tfrac{\partial x}{\partial\hat{x}} & \tfrac{\partial y}{\partial\hat{x}} & 0 & 0 & 0 & 0 & 0 & 0 & 0 & 0 & 0 & 0 & 0 & 0\\
0 & \tfrac{\partial x}{\partial\hat{y}} & \tfrac{\partial y}{\partial\hat{y}} & 0 & 0 & 0 & 0 & 0 & 0 & 0 & 0 & 0 & 0 & 0 & 0\\
0 & 0 & 0 & 1 & 0 & 0 & 0 & 0 & 0 & 0 & 0 & 0 & 0 & 0 & 0\\
0 & 0 & 0 & 0 & \tfrac{\partial x}{\partial\hat{x}} & \tfrac{\partial y}{\partial\hat{x}} & 0 & 0 & 0 & 0 & 0 & 0 & 0 & 0 & 0\\
0 & 0 & 0 & 0 & \tfrac{\partial x}{\partial\hat{y}} & \tfrac{\partial y}{\partial\hat{y}} & 0 & 0 & 0 & 0 & 0 & 0 & 0 & 0 & 0\\
0 & 0 & 0 & 0 & 0 & 0 & 1 & 0 & 0 & 0 & 0 & 0 & 0 & 0 & 0\\
0 & 0 & 0 & 0 & 0 & 0 & 0 & \tfrac{\partial x}{\partial\hat{x}} & \tfrac{\partial y}{\partial\hat{x}} & 0 & 0 & 0 & 0 & 0 & 0\\
0 & 0 & 0 & 0 & 0 & 0 & 0 & \tfrac{\partial x}{\partial\hat{y}} & \tfrac{\partial y}{\partial\hat{y}} & 0 & 0 & 0 & 0 & 0 & 0\\
0 & 0 & 0 & -B^{1,1}_{12} & 0 & 0 & B^{1,1}_{12} & 0 & 0 & B^{1,1}_{11} & 0 & 0 & 0 & 0 & 0\\
-B^{1,2}_{12} & 0 & 0 & 0 & 0 & 0 & B^{1,2}_{12} & 0 & 0 & 0 & B^{1,2}_{11} & 0 & 0 & 0 & 0\\
-B^{1,3}_{12} & 0 & 0 & B^{1,3}_{12} & 0 & 0 & 0 & 0 & 0 & 0 & 0 & B^{1,3}_{11} & 0 & 0 & 0\\
0 & 0 & 0 & 0 & -\beta^{1,x} & -\beta^{1,y} & 0 & \beta^{1,x} & \beta^{1,y} & 0 & 0 & 0 & B^{2,1}_{11} & 0 & 0\\
0 & -\beta^{2,x}  & -\beta^{2,y} & 0 & 0 & 0 & 0 & \beta^{2,x} &\beta^{2,y} & 0 & 0 & 0 & 0 & B^{2,2}_{11} & 0\\
0 & -\beta^{3,x} & -\beta^{3,y} & 0 & \beta^{3,x} & \beta^{3,y} & 0 & 0 & 0 & 0 & 0 & 0 & 0 & 0 & B^{2,3}_{11}\\
\end {bmatrix}
  }
\end{equation}
}
for $V$, where $\beta$ is as defined in~\eqref{eq:beta}.

%% file: discretisation.tex
\section{Discretisation}\label{sec:discretisation}

We now describe the discretisations of the Hamiltonian system \eqref{eq:hamiltonseqs_reduced}
using a function space introduced in the previous section.
Smooth solutions to \eqref{eq:hamiltonseqs_reduced:u} generates the following curve of
diffeomorphisms:
\begin{equation}
\dot{\varphi}_t = u_t \circ\varphi_t, \quad \varphi_0=\idop,
\end{equation}
where the domain of $\varphi_t$ is $\Omega_0$. This subsumes the left action on the curve $q_0$ in 
\eqref{eq:hamiltonseqs_reduced:q}. Our approach is therefore to solve 
\eqref{eq:hamiltonseqs_reduced:u} for an \emph{outer metric} in tandem 
with integrating the diffeomorphism defined over the entire domain
and moving the mesh, thereby automatically providing a solution to a
discrete analogue of \eqref{eq:hamiltonseqs_reduced:q}. We denote by
$\mathcal{T}_0$ denote a shape-regular, quasi-uniform triangulation
of the template domain $\Omega_0$. Let $\meshskel$ denote the \emph{mesh
skeleton} of $\mathcal{T}_0$ and $\meshskelint$ the subset of $\meshskel$
whose elements do not intersect $\partial\Omega_0$. We place the following
assumption on the initial triangulation $\mathcal{T}_0$:
\begin{assumption}\label{assumption:embedded}
$\mathcal{T}_0$ is constructed such that the range of $q_0$ is described by a subset of $\meshskelint$.
\end{assumption}

Using the definition in \eqref{eq:wxK} we define the vector-valued
Wu-Xu space defined over $\Omega_0$:
\[
V(\Omega_0) = \{ v \in L^2(\Omega)^2 \,|\, v_i|_K \in \wx(K),\, K \in \mathcal{T}_0,\, i=1,2\}.
\]
Further, let $0=t_0, t_{\Delta T}, \ldots t_{T-1}=1$ denote $T$ uniformly distributed
points and use an Euler discretisation of the time derivate in \eqref{eq:varphi},
where we let $\varphi_{t_k} \in V(\Omega_0)$, $u_{t_k} \in V(\Omega_0)$:
\begin{equation}\label{eq:varphi:disc}
\varphi_{t_{k+1}} = \varphi_{t_k} + u_{t_k}\circ\varphi_{t_k} \Delta T.
\end{equation}
For sufficiently small $\Delta T$, $\varphi_{t_k}$ is a diffeomorphism
of $\Omega$ \cite{allaire2007conception}. Using the notation $\Omega_{t_k} =
\varphi_{t_k} \circ \Omega_0$, $V(\Omega_{t_k}) := \{f \,|\, f \circ \varphi_{t_k}\inv
\in V(\Omega_0)\}$ and by noting that $q_{t_k} = \varphi_{t_k} \circ q_0$ we obtain
a discrete analogue of \eqref{eq:hamiltonseqs_reduced} where $\hat u_{t_k}
\in V(\Omega_{t_k})$:
\begin{subequations}\label{eq:hamiltonseqs_reduced:disc}
\begin{align}
&  a_{\Omega_{t_k}}(\hat u_{t_k}, \hat v) = \int_{S^1} \nabla\varphi_{t_k}\invtransp \circ q_0
\mathbf{n}_{q_0} \tilde{p}_0 \cdot \hat v \circ q_0 \diff{s}, & \forall \hat v \in V(\Omega_{t_k}), \label{eq:hamiltonseqs_reduced:u:disc}\\
& \varphi_{t_{k+1}} = \varphi_{t_k} + \hat u_{t_k} \Delta T,
\end{align}
\end{subequations}
for $k=0,\ldots, T-1$, where $\varphi_{t_k}\circ\partial\Omega_0 = \idop$ owing to
the homogeneous Dirichlet boundary conditions implied by \eqref{eq:hamiltonseqs_reduced:u:disc}.
At each time step $k$ after the solution of \eqref{eq:hamiltonseqs_reduced:u:disc}, 
the mesh is moved according to \eqref{eq:varphi:disc} upon which the equation \eqref{eq:hamiltonseqs_reduced:q:disc}:
\begin{equation}
q_{t_{k+1}} = q_{t_k} + u_{t_k}\circ q_{t_k} \Delta T, \label{eq:hamiltonseqs_reduced:q:disc}
\end{equation}
is automatically satisfied. The underlying coordinate field of the mesh itself
is chosen to be a Lagrange subspace of $V(\Omega_0)$, so that the map 
$q_0 \mapsto \varphi_{t_k} \circ q_0$ is a diffeomorphism. At $k=0$, the assembly of the right-hand
side is in practice done by integration over $q_0\circ S^1$, which means that we can supply an 
initial ``momentum'' signal $\mathbf{p}_0\in \tilde{p}_0\circ q_0\inv\in L^2(q_0\circ S^1)$ (now
defined over initial curve) to encode the entire geodesic flow of $\varphi$, and thereby of the 
embedded curve. Figure \ref{fig:ivps} show examples of forward integration of this system for
various $\mathbf{p}_0$ and $q_0\circ S^1$ (the initial meshes were generated using \texttt{gmsh}
\cite{geuzaine2009gmsh}). Note that the norm of the velocity present in
\eqref{eq:hamiltonseqs_reduced} is confined to certain energy levels determined
by the initial momentum as the system is integrated. In the fully discrete
analogue we can only hope to establish approximate conservation of the Hamiltonian.
The importance of this nebulous since we only integrate over fixed time intervals,
and is subject to future work.\\

The computational cost of integrating \eqref{eq:hamiltonseqs_reduced:disc} is dominated
by the inversion of the discrete bilinear form. Mesh-based methods readily facilitate 
parallel computations (e.g. matrix-vector products in a Krylov subspace method), which
along with preconditioning strategies are competitive with fast multipole methods. They
also offer flexibility in choosing bilinear form (which can be altered according to an 
informed modelling choice or application). Finally, mesh adaptivity is also an
option. For the application at hand a \emph{graded} mesh with a fine resolution in
the vicinity of the curve and coarser elements closer to the boundary can both 
increase accuracy and the computational burden of the method.

\begin{figure}[ht!]
    \centering
        \includegraphics[width=0.25\textwidth]{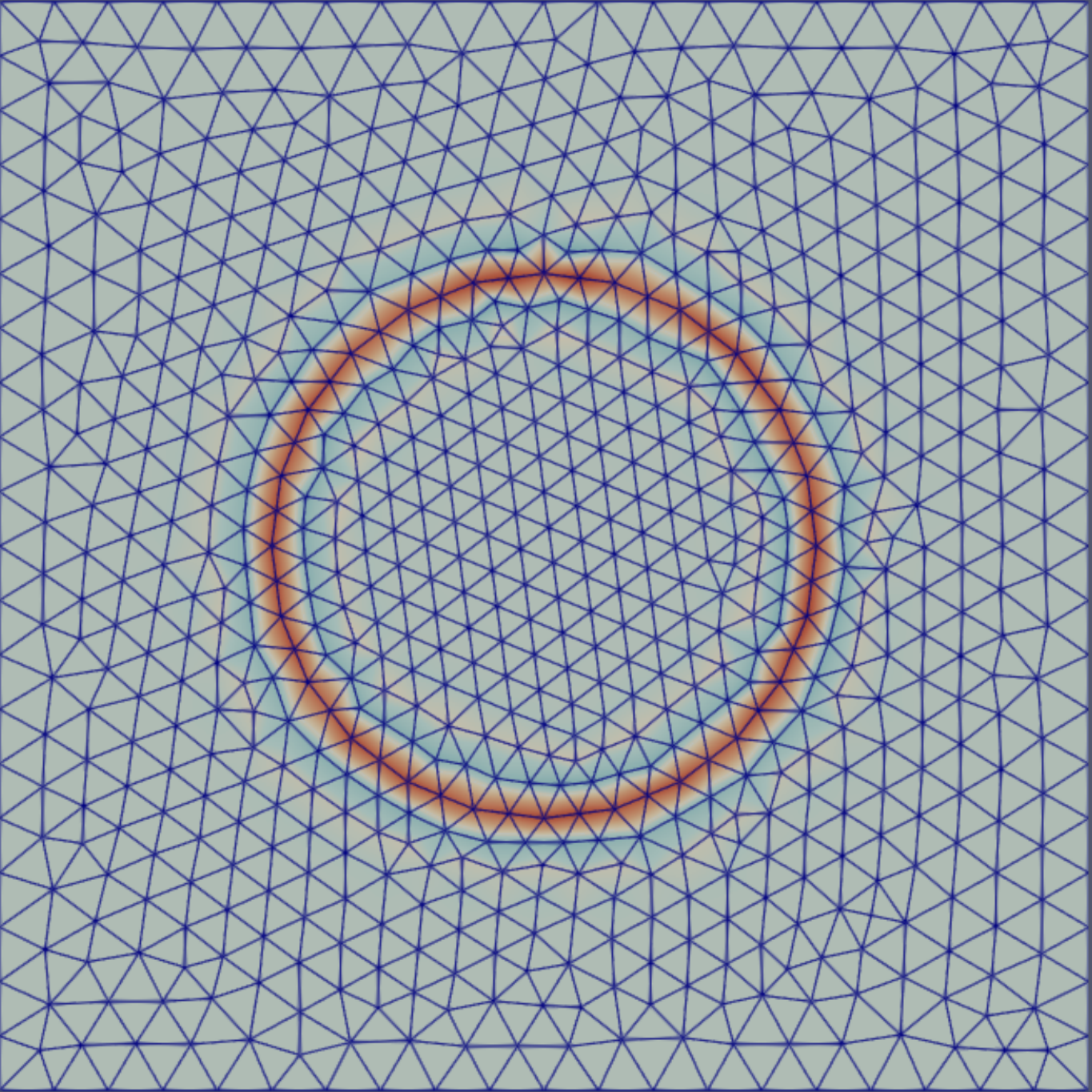}
        \includegraphics[width=0.25\textwidth]{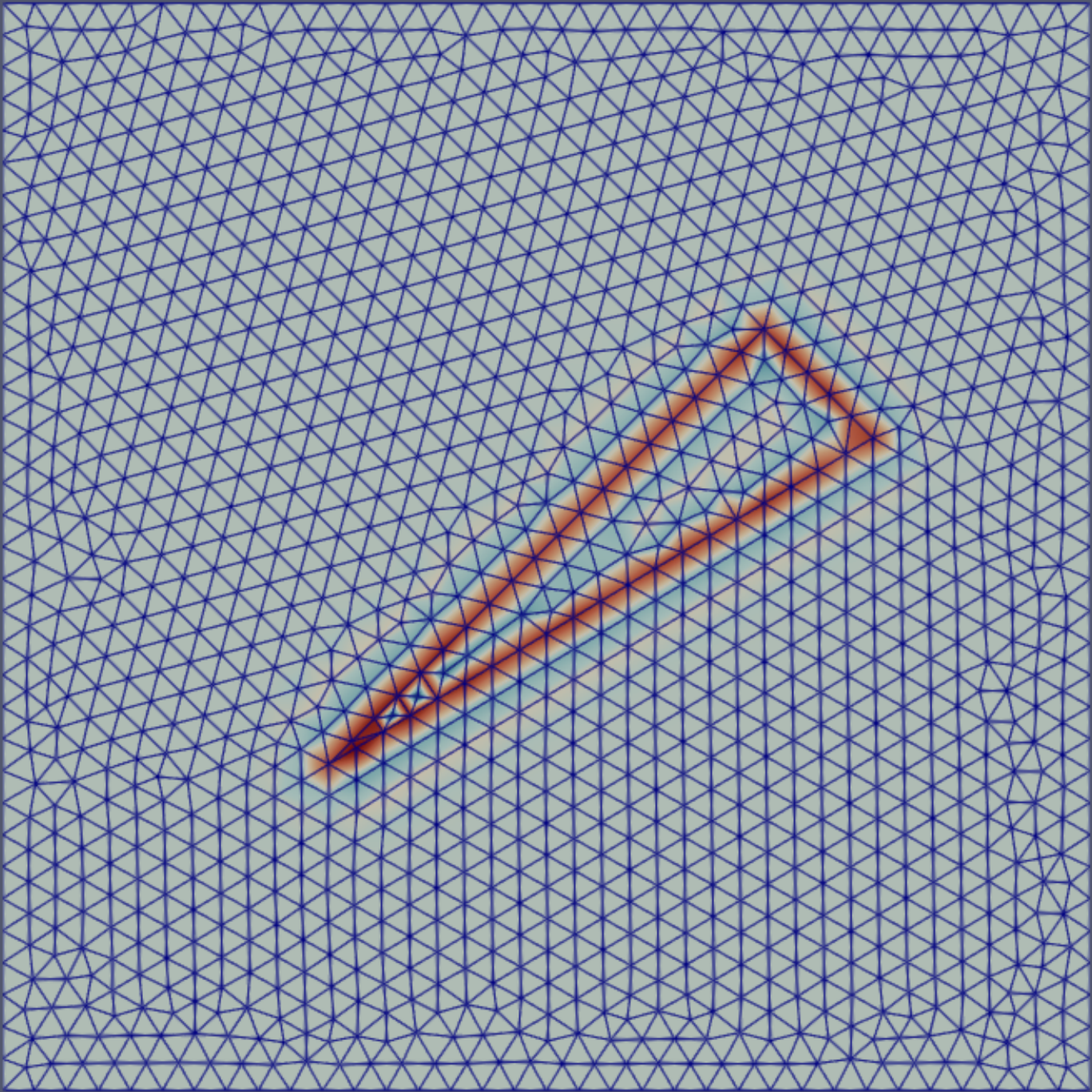}\\[2pt]
        \includegraphics[width=0.25\textwidth]{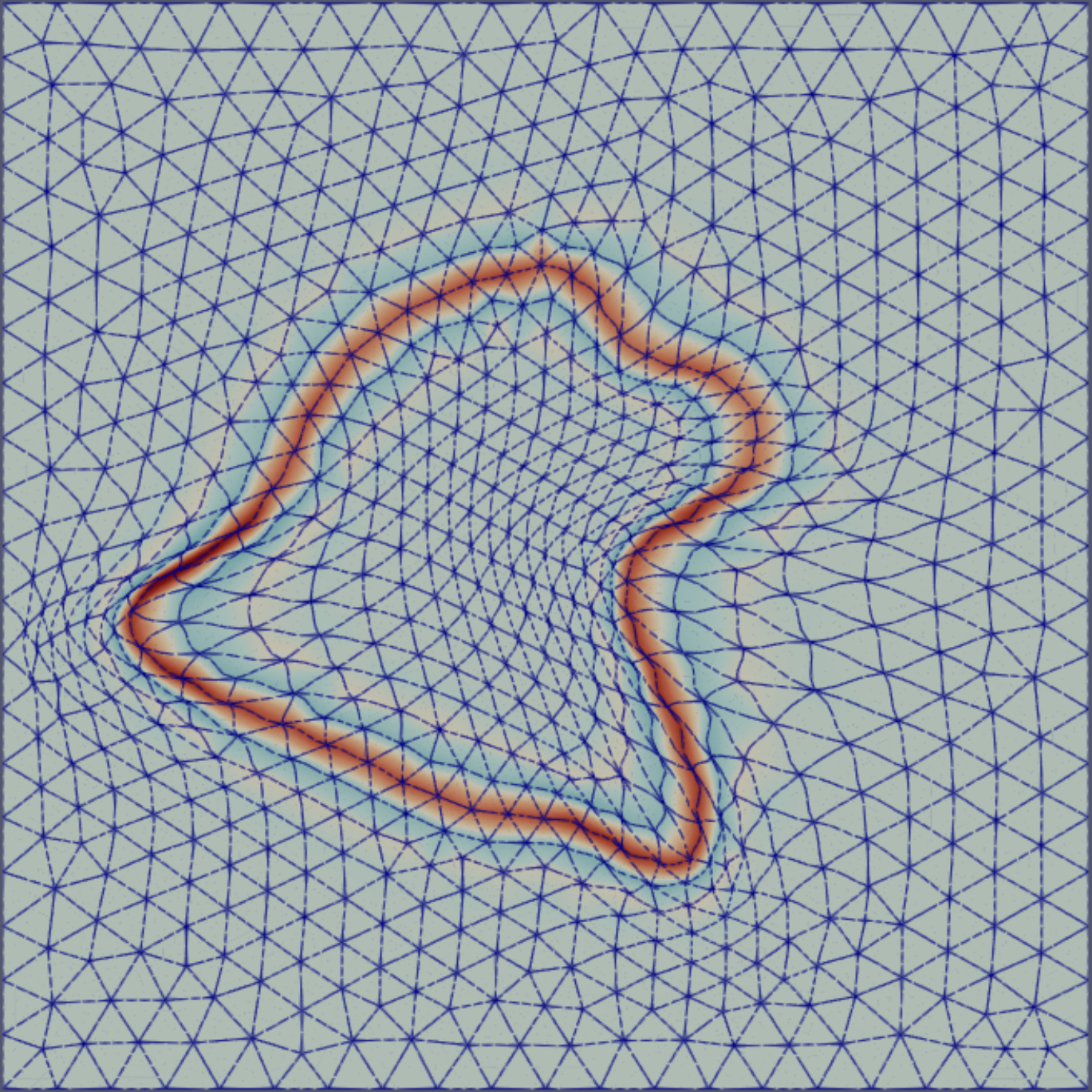}
        \includegraphics[width=0.25\textwidth]{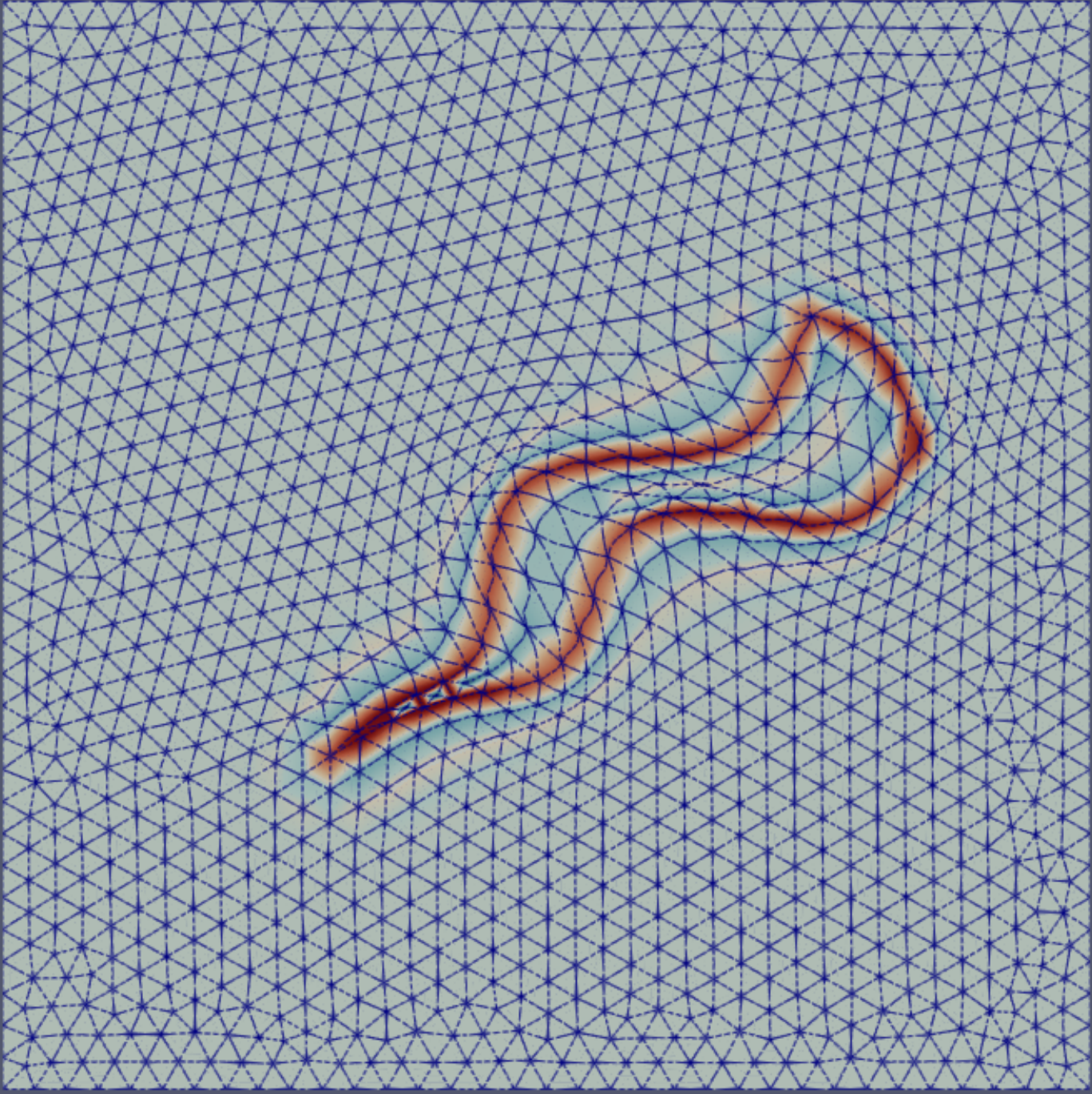}
        \caption{Top row: template domains $\Omega_0$ with curves highlighted. 
        Bottom row: $\varphi_1 \circ \Omega_1$ computed by integrating \eqref{eq:hamiltonseqs_reduced:disc}
        from two templates with different initial momenta.}
    \label{fig:ivps}
\end{figure}

%% file: inverse_problem.tex
\section{Inverse problem}\label{sec:inverse_problem}

We now consider the matching problem using the data misfit functional in \eqref{eq:mismatch}.
We wish to estimate the momentum $\momentum_0 := \tilde{p}_0\circ q_0\inv\in\MomentumSpace$
that generates the curve $t\mapsto q_t$. That is, $\momentum_0$ is the momentum object
defined on the computational domain $q_0 \circ S^1$. We drop explicit dependence on the template 
as it remains fixed during computation as well as the time dimension of the initial
momentum. To ease the notation we use boldface $\bq$ to represent the smoothed version of the
indicator function on the interior of a curve $q$ the i.e.:
\[
\bq = \mathcal{C}\inv\mathbb{1}_q.
\]
We define the \emph{forward operator}:
\begin{equation}\label{eq:forward_operator}
\momentum \mapsto \mathcal{F}(\momentum) = \bq := \mathcal{C}\inv\mathbb{1}_{q_1},
\end{equation}
where $q_1$ is the solution at $t=1$ given by solving \eqref{eq:hamiltonseqs_reduced:disc}
using $q_0$ and $\momentum$ as initial conditions, i.e. the time-$1$ flow map of the initial
curve. Given a target shape $\qtarget$ the inverse problem of interest is therefore to
recover the momentum $\momentum^*$ such that:
\[
\bqtarget \approx \mathcal{F}(\momentum^*) + \xi,
\]
where the noise $\xi \in \mathcal{N}(0, \mathcal{C})$ is a Gaussian measure with mean
zero and covariance operator $\mathcal{C}$ defined later on. 

To tackle this inverse problem, we use an Ensemble Kalman iteration. We let $N$ denote the number
of ensemble members and let $\momentum^j$, $j=1,\ldots,N$ denote the momenta corresponding
to ensemble member $j$. The ensemble mean momentum and the mean predicted shape are:
\begin{align}\label{eq:means}
\AvgMomentum := N\inv \sum_{j=1}^N \momentum_j, \qquad \AvgPoint := N\inv \sum_{j=1}^N \bq^j,
\end{align}
where $\bq^j = \mathcal{C}\inv\mathbb{1}_{\mathcal{F}(\momentum^j)}$. The Kalman update operator is defined by:
\begin{align}\label{eqs:kalmanops}
& \kalmanp = \cov_{\MomentumSpace Q} [\cov_{QQ} + \xi I]^{-1},
\end{align}
where $\xi$ is a regularisation parameter described later and $I$ is the identity operator.
The actions above are given by:
\begin{subequations}
\begin{align}
& \cov_{QQ}[\cdot] = \frac{1}{N - 1}\sum_{j=1}^N (\bq^j - \AvgPoint) \langle \bq^j   - \AvgPoint, \cdot \rangle_{L^2},\label{CovQQ}\\
& \cov_{\MomentumSpace Q}[\cdot] = \frac{1}{N - 1}\sum_{j=1}^N (\momentum^j - \AvgMomentum) \langle \bq^j   - \AvgPoint, \cdot \rangle_{L^2}.\label{CovPQ}
\end{align}
\end{subequations}
The data misfit at iteration $k$ of the EKI is defined as:
\begin{equation}\label{eq:mismatch:C}
\mathfrak{E}^k = \| \bqtarget - \AvgPoint \|_{L^2(\Omega)}^2.
\end{equation}
The prediction and analysis steps are summarised below:
\begin{enumerate}
    \item \emph{Prediction}: For each ensemble member $j$, compute $\bq^j = \mathcal{C}\inv\mathbb{1}_{\mathcal{F}(\momentum^j)}$ and the average $\AvgPoint$ using \eqref{eq:means}.
    \item \emph{Analysis}: Update each ensemble momentum:
        \[
        \momentum^{j+1} = \momentum^j + \kalmanp (\bqtarget - \bq^j).
        \]
\end{enumerate}

%% file: numerical_results.tex
\section{Numerical experiments}\label{sec:applications}

We present numerical experiments showing that the ensemble Kalman
inversion (EKI) algorithm is able to find suitable approximations of a
target given a random initial ensemble. Section \ref{sec:synthetic_data}
describes how we generate the synthetic target that we will use as 
matching targets. Section \ref{sec:experimental_setup} summarises 
the parameters that we have chosen to in our experiments to match
the synthetic data, and section \ref{sec:results} contains the
numerical results.

\subsection{Synthetic data}\label{sec:synthetic_data}

For simplicity we fix the template curve throughout our experiments
and choose the unit circle. The initial mesh is that shown in the
top left vignette of Figure \ref{fig:ivps}. The computation domain $\Omega_0$
is a triangulation of $[-10,10]^2$ with mesh resolution\footnote{This is
the maximum diameter $h_K$ of any triangle $K$ in the triangulation.}
$h=1$. We have taken $\alpha=0.5$ in \eqref{eq:bilinear_form}, $T=15$ time steps and 
have used piecewise constant finite elements on the mesh skeleton to represent
$\MomentumSpace$ (although we compute only with functions supported over 
the submanifold $q_0 \circ S^1\subset\meshskel$). We use the forward
operator described previously to generate synthetic targets for 
this set of parameters. Applying the forward operator $\mathcal{F}$ to the
momenta in \eqref{eq:manufactured_momentum} below we produce the targets
seen in Figure \ref{fig:synthetic_targets}.
\begin{subequations}\label{eq:manufactured_momentum}
\begin{align}
& \momentum^\dagger_\text{contract} = -1.38\pi,\\
& \momentum^\dagger_\text{squeeze} = \begin{cases} 
      0.83\pi e^{-y^2/5} & x< -0.3 \\
      \frac{5}{3} \pi \sin(x / 5)|y| & \textnormal{otherwise}
   \end{cases},\\
& \momentum^\dagger_\text{star} = 2.6\pi \cos(2\pi x/5),\\
& \momentum^\dagger_\text{teardrop} = \begin{cases} 
      -3\pi\,\textnormal{sign}(y) & y<0 \\
      3\pi e^{-x^2/5} & \textnormal{otherwise}
   \end{cases}.
\end{align}
\end{subequations}

\begin{figure}
    \centering
    \includegraphics[scale=0.32]{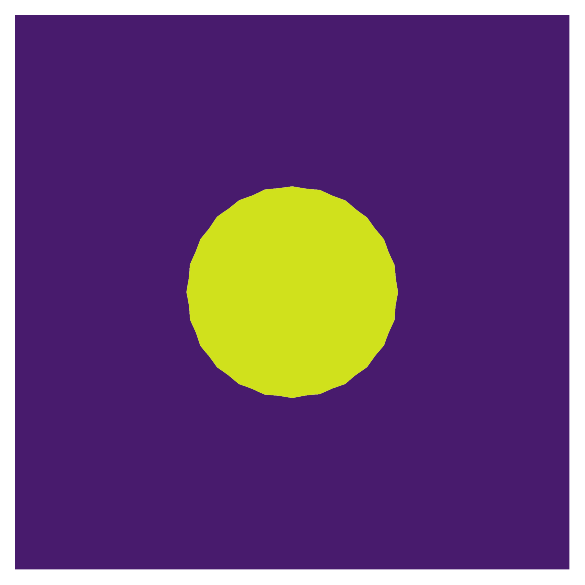}
    \includegraphics[scale=0.32]{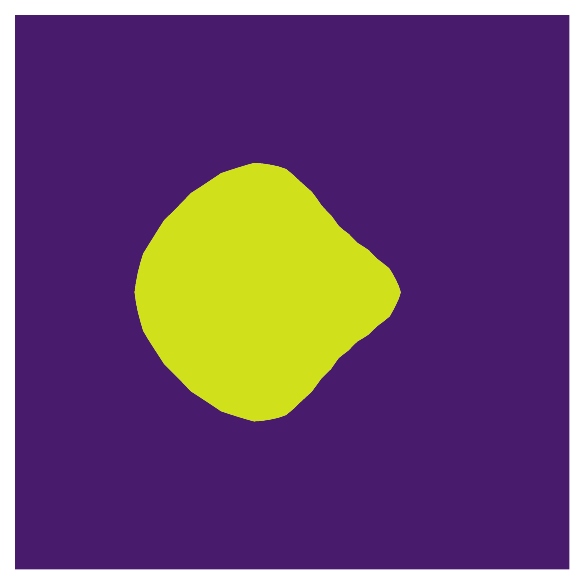}
    \includegraphics[scale=0.32]{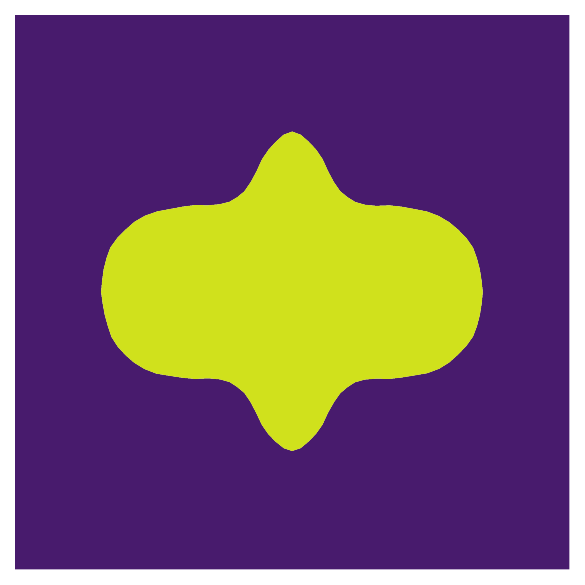}
    \includegraphics[scale=0.32]{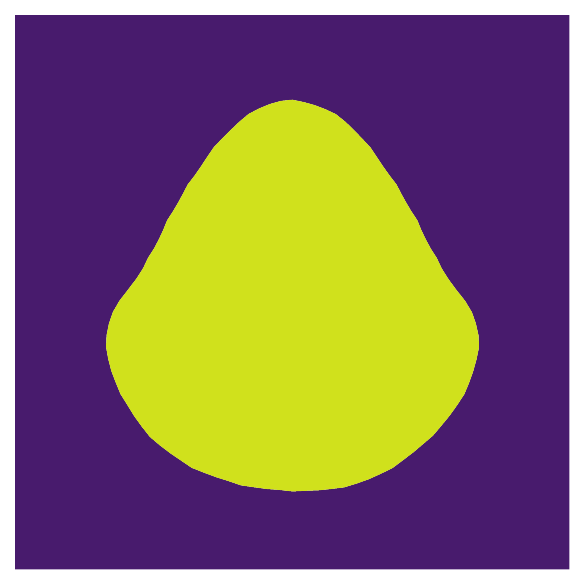}
    \caption{Synthetic targets generated using the momentum in \eqref{eq:manufactured_momentum}. A simple contraction of the shape, a squeezed figure with a right bias, a star and a teardrop shape).}
    \label{fig:synthetic_targets}
\end{figure}

With $\momentum^\dagger$ we associate the following relative error at each iterate $k$:
\begin{equation}\label{eq:relative_error}
\mathcal{R}^k = \| \bar{\momentum}^k  - \momentum^\dagger \|_{L^2(q_0\circ S^1)} / \| \momentum^\dagger \|_{L^2(q_0\circ S^1)}.
\end{equation}
The \emph{consensus deviation} $\mathcal{S}^k$ of an ensemble at iteration
$k$ in equation \eqref{eq:consensus} is defined below:
\begin{equation}\label{eq:consensus}
\mathcal{S}^k = N\inv\sum_{j=1}^N \| \momentum^{j,k} - \bar{\momentum}^k\|_{L^2(q_0\circ S^1)},
\end{equation}
where $\momentum^{j,k}$ denotes the momentum of ensemble member $j$ at
iteration $k$. This quantity is a useful diagnostic which measures the 
information remaining in the ensemble after iteration $k$. Since EKI
relies on estimates of the forecast covariance, consensus is reached when
$\mathcal{S}^k$ approaches zero, at which point the algorithm can be
stopped.\\

In all our simulations we invert the system in \eqref{eq:hamiltonseqs_reduced:u:disc}
using the direct solver MUMPS \cite{amestoy2000mumps}; investigating a preconditioned 
iterative solver is subject to future work. For details on the validation of the
implementation of the Wu-Xu element in Firedrake, see \cite[Appendix B]{bock2020inverse}
and the Zenodo entry \cite{zenodo/Firedrake-20200518.0}.

\subsection{Experimental setup}\label{sec:experimental_setup}

We now describe the setup we have used to test the EKI.
Firstly, we have taken $T=10$ and $\alpha=1$
so the parameters differ from those used to generate the synthetic
targets. Recall that EKI requires an initial ensemble, in this
case of momenta. The basis coefficients of the momenta was sampled 
from a uniform distribution over the interval $[-25,25]$, with different
realisations for each ensemble member. The parameter $\xi$ in
\eqref{eqs:kalmanops} determines the ratio between the influence 
of the prediction covariance on the Kalman gain. We set $\xi=10^{-3}$
in \eqref{eqs:kalmanops}, although adaptive tuning of this
parameter to avoid overfitting is possible; an early termination rule
is suggested in \cite[Equation 10]{iglesias2016regularizing}.
We choose $\mathcal{C} = (\idop - \kappa \Delta)\inv$, $\kappa = 10$
in \eqref{eq:mismatch:C}, as this smoothes out the mismatch sufficiently
for our computational domain.
The quality of the matching is directly related to the size and
variance of the ensemble as the solution is sought as a linear
combination of its members. We conduct experiments
for two ensemble sizes, $N=20$, $N=40$ and $N=80$. These
were chosen since, with the parameter set as above, $\dim \MomentumSpace
= 48$ in order to develop an understanding of how EKI performs when
the ensemble size is smaller than, similar to and larger than the dimension
of the state, while still keeping the overall computational cost such
that the experiments can be done in a reasonable amount of time.
The case where $N<<\dim\MomentumSpace$ is the de facto situation
for ensemble methods as the MC approximation allows for a computationally 
feasible method. In general, small ensemble sizes can lead to filter 
inbreeding (the forecast covariance is underestimated), filter divergence (the gain
does not adequately correct the ensemble), or spurious correlations
 \cite{petrie2008localization}. We comment on each of these later.

\subsection{Results}\label{sec:results}

We have run EKI $10$ times for each value of $N$ with different draws
for the initial ensemble to assess the robustness of the algorithm
with respect to the starting point. Figure \ref{fig:shapes} shows examples of the numerical
results we obtain for curve matching using EKI. Note that only five 
iterations of EKI were necessary to produce the results shown 
in this section to reach a relative tolerance below $3\%$. Qualitatively
a larger ensemble size leads to a better match, which is to be expected.
Ensemble methods such as the EKI offer a practical advantage to gradient
methods given their inherent parallelisability. Indeed, the \emph{prediction} 
step discussed in Section \ref{sec:inverse_problem} can be done in parallel
for each ensemble member. We therefore start $N$ processes corresponding to
each member, and used a Message Passing Interface
(MPI) \cite{gropp1999using} implementation to exchange information between
them (the MPI \emph{reduce} operation, specifically). Thanks to this
parallelisation, five iterations of EKI takes less than two minutes for
$N=20$, five minutes for $N=40$ and 14 minutes for $N=80$ on a 2021 MacBook
Pro\footnote{Apple M1 Pro chip, 16 GB of memory.}.\\ 

\begin{figure}
    \centering
    \includegraphics[scale=0.32]{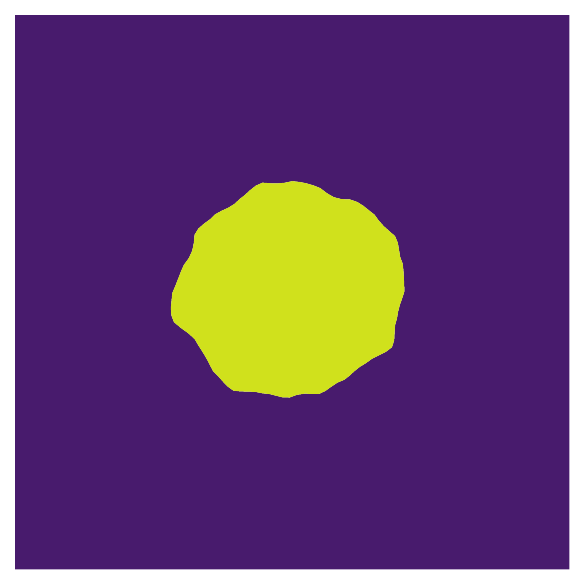}
    \includegraphics[scale=0.32]{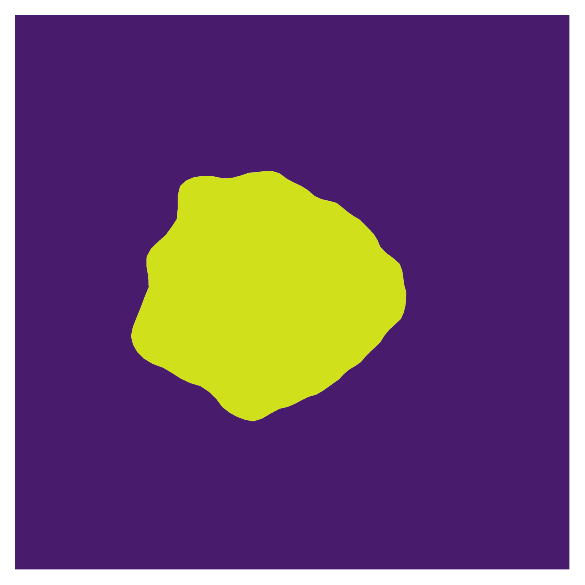}
    \includegraphics[scale=0.32]{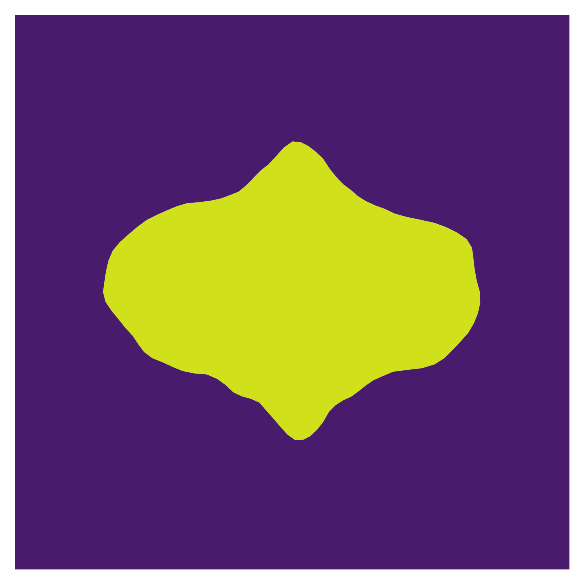}
    \includegraphics[scale=0.32]{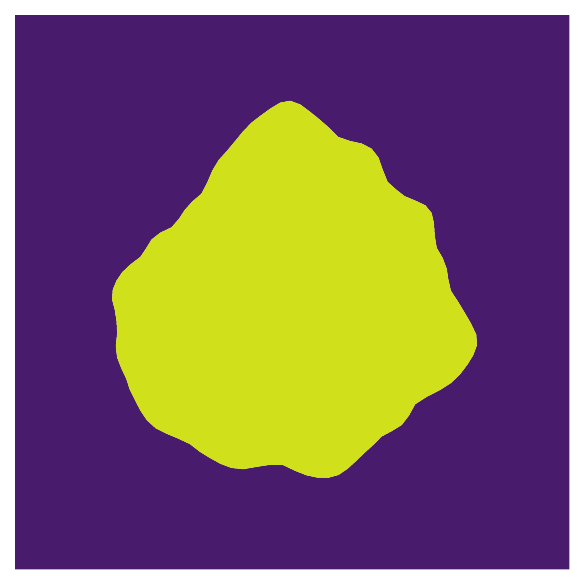}\\
    \includegraphics[scale=0.32]{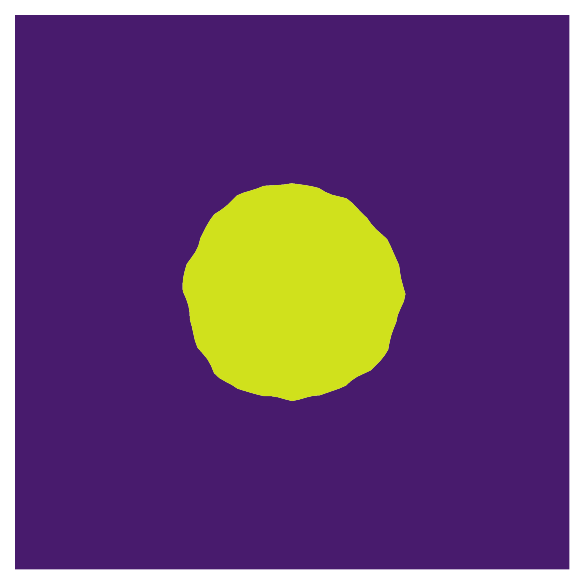}
    \includegraphics[scale=0.32]{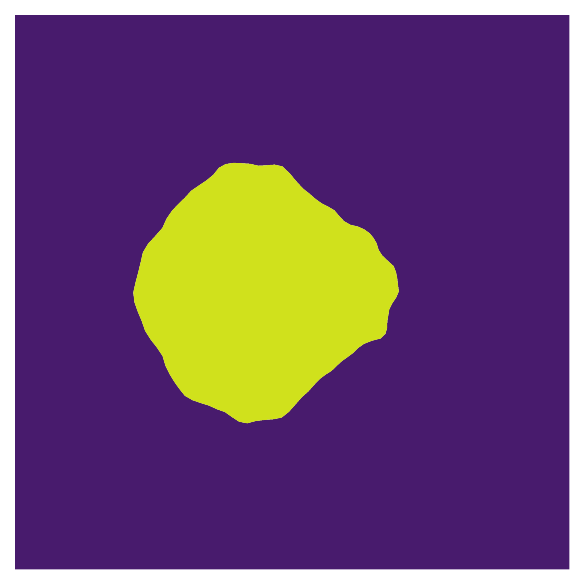}
    \includegraphics[scale=0.32]{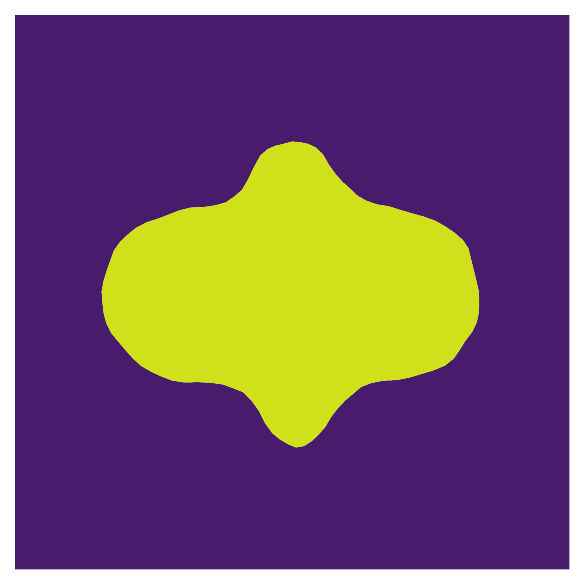}
    \includegraphics[scale=0.32]{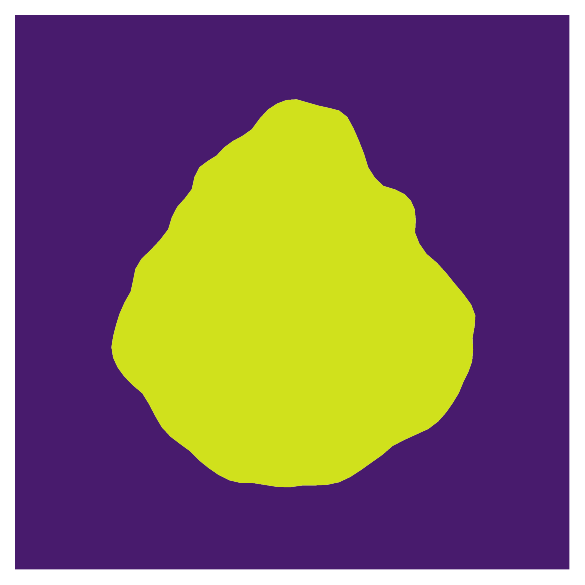}\\
    \includegraphics[scale=0.32]{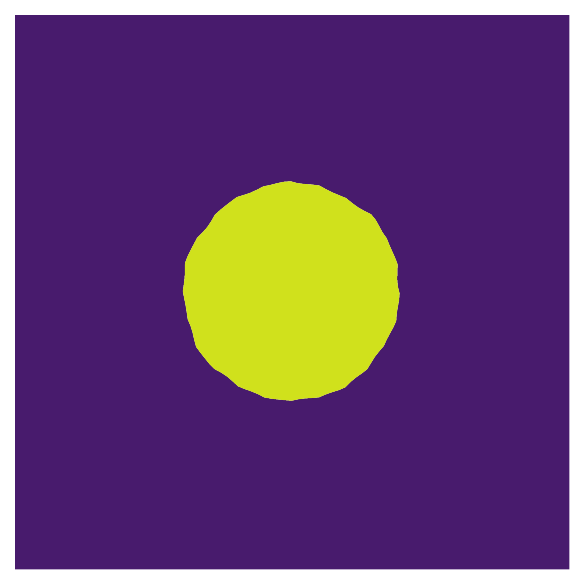}
    \includegraphics[scale=0.32]{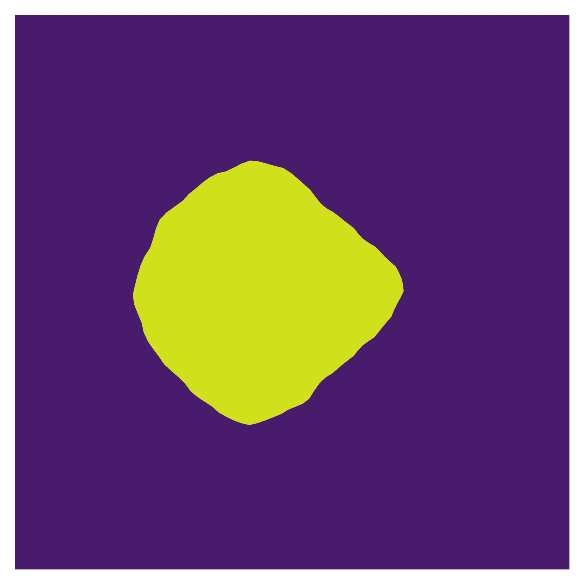}
    \includegraphics[scale=0.32]{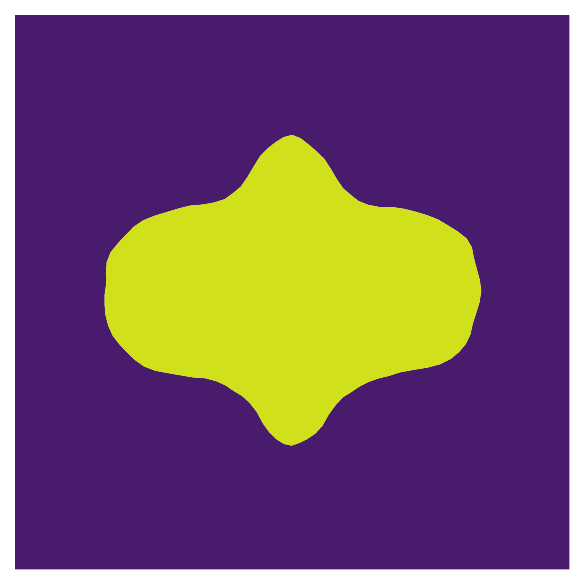}
    \includegraphics[scale=0.32]{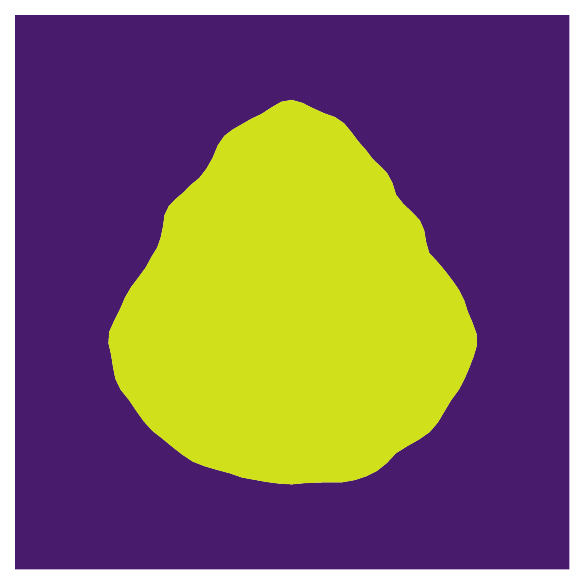}
    \caption{Final reconstruction of the targets in figures \ref{fig:synthetic_targets}
    produced by EKI for $N=20$ (top row, $N=40$ (middle row) and $N=80$ (bottom row).}
    \label{fig:shapes}
\end{figure}

Figure \ref{fig:relative_error}, \ref{fig:data_misfits} and 
\ref{fig:consensus} show the relative errors, data misfits and
consensus deviations for our experiments across the selected
targets and ensemble sizes. These all decrease at various rates in the early
iterations after which they stagnate. As the Kalman gain corrects the
ensemble members, and therefore the motion of their respective curves,
the data misfit decreases meaning that each member improves its prediction.
This increases consensus in the ensemble, which explains what is
seen in figure \ref{fig:consensus}. We notice from the data misfits and
the momentum consensus that higher values of $N$ provides a more accurate 
approximation of the true momentum, which explains the accuracy of the
matches seen in figure \ref{fig:shapes}. Note that the relative error, which
is a surrogate for posterior consistency, also decreases (albeit not with
a clear pattern across shapes and ensemble sizes). Since the forward operator
in use is heavily nonlinear, theoretical convergence results are not readily 
obtained at this stage. We highlight that the Kalman gain is very efficient
in correcting the ensemble momenta, with the consensus decreasing exponentially
in the early stages of the algorithm. We comment on this later. The conclusion
is therefore that even at a modest ensemble size, EKI performs well. It is 
not certain that the same behaviour that we see above (i.e. few iterations are
needed) will scale with $N$ and the size of the problem, but the results are
promising for research in this direction.\\

Higher values of $\xi$ were found to slow the convergence of the algorithm compared to the
selected value, which is consistent with the behaviour for landmark-based EKI
\cite{bock2021learning}, and we do not comment on it further. We noticed that
the value of $\kappa$ also influences the convergence of the EKI; for small
values the operator $\idop - \kappa\Delta$ approaches the identity, and since
the mismatch $X$ is computed from point evaluations of the finite element
function, information can be lost if the grid is not sufficiently refined.
A larger value of $\kappa$ ``spreads out'' the mismatch which improves convergence
for coarser grids.\\

\begin{figure}
    \centering
    \includegraphics[scale=0.20]{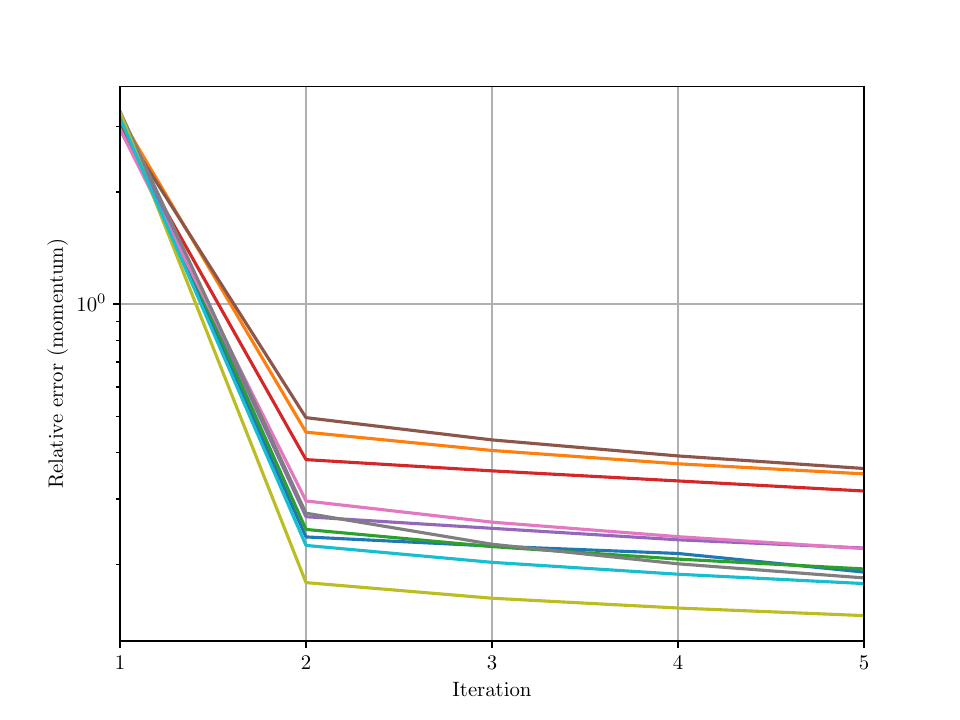}
    \includegraphics[scale=0.20]{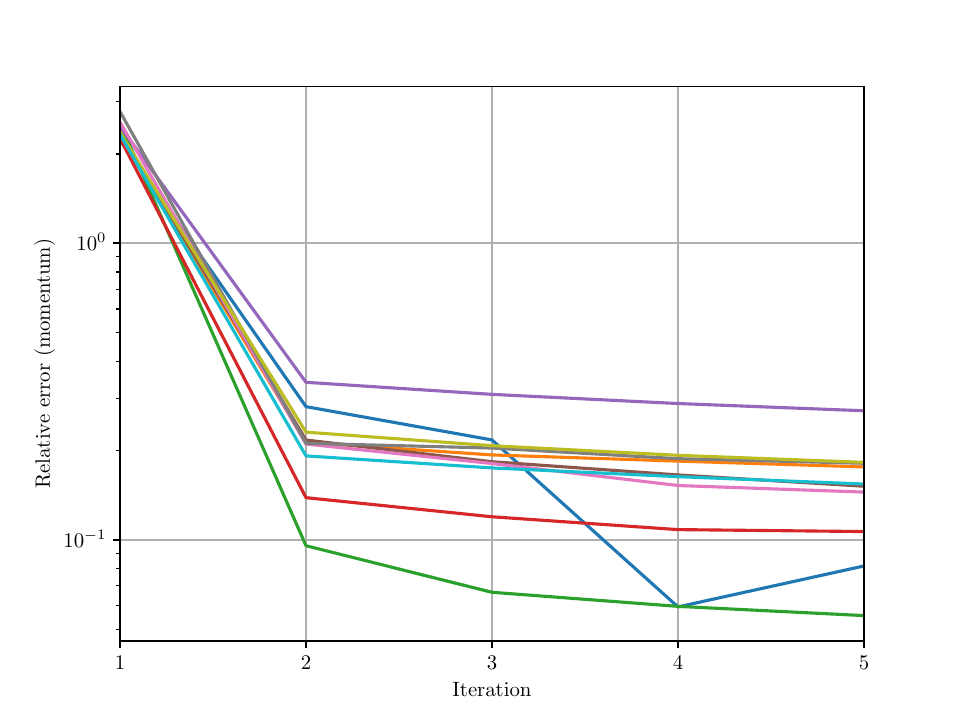}
    \includegraphics[scale=0.20]{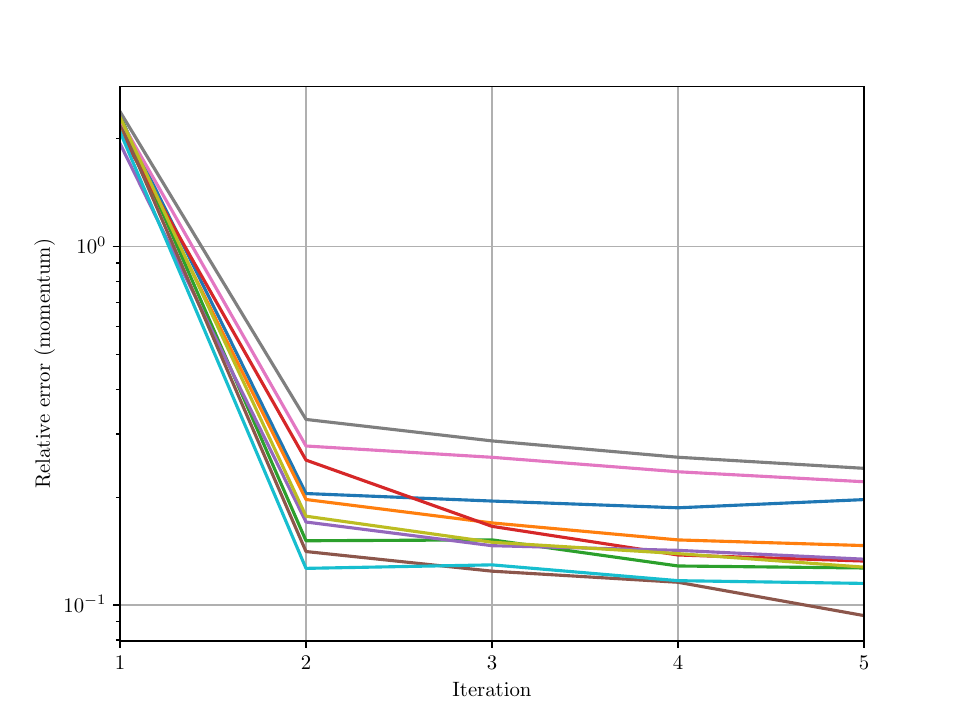}
    \includegraphics[scale=0.20]{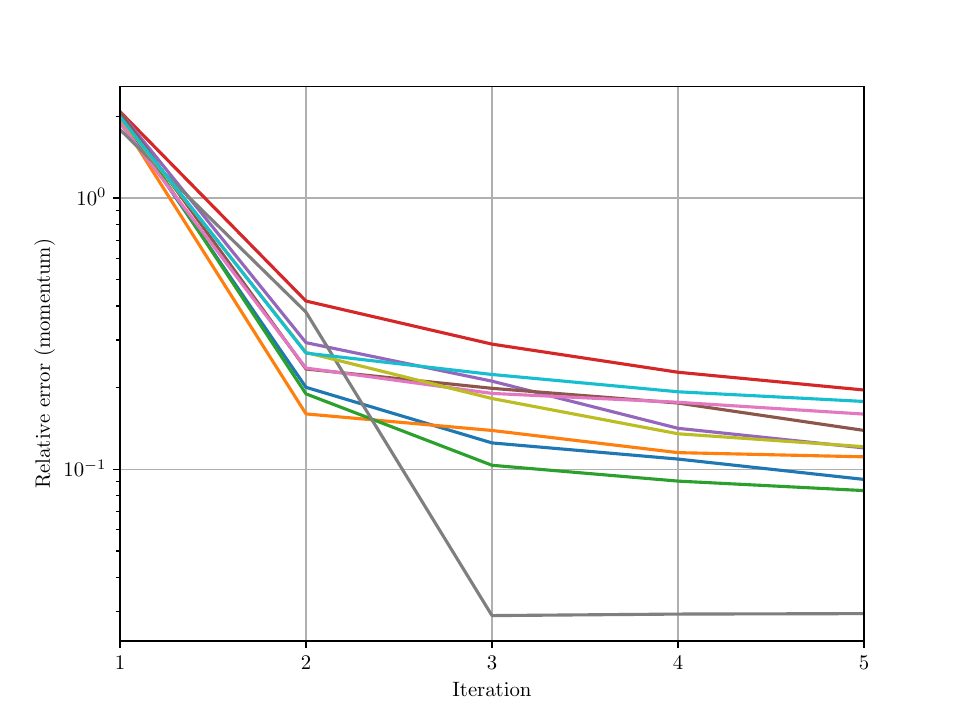}\\
    \includegraphics[scale=0.20]{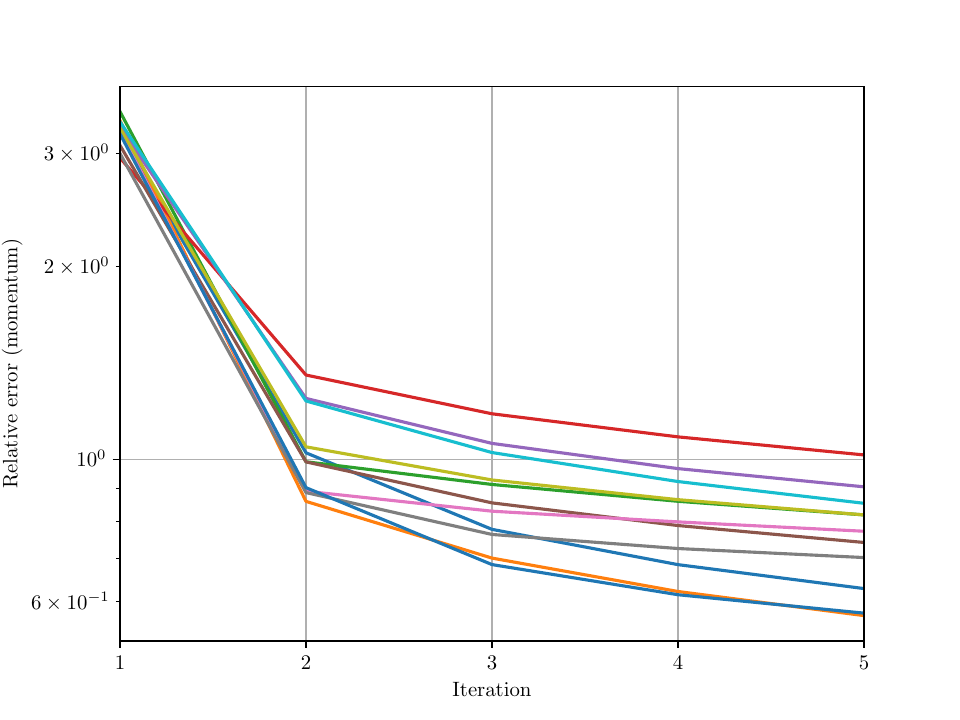}
    \includegraphics[scale=0.20]{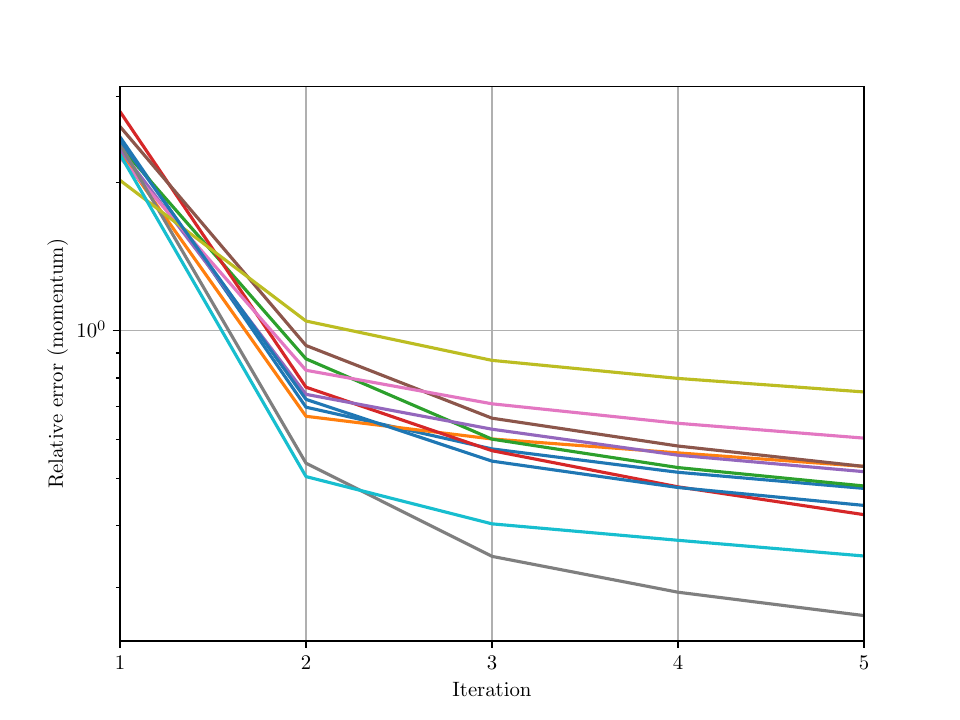}
    \includegraphics[scale=0.20]{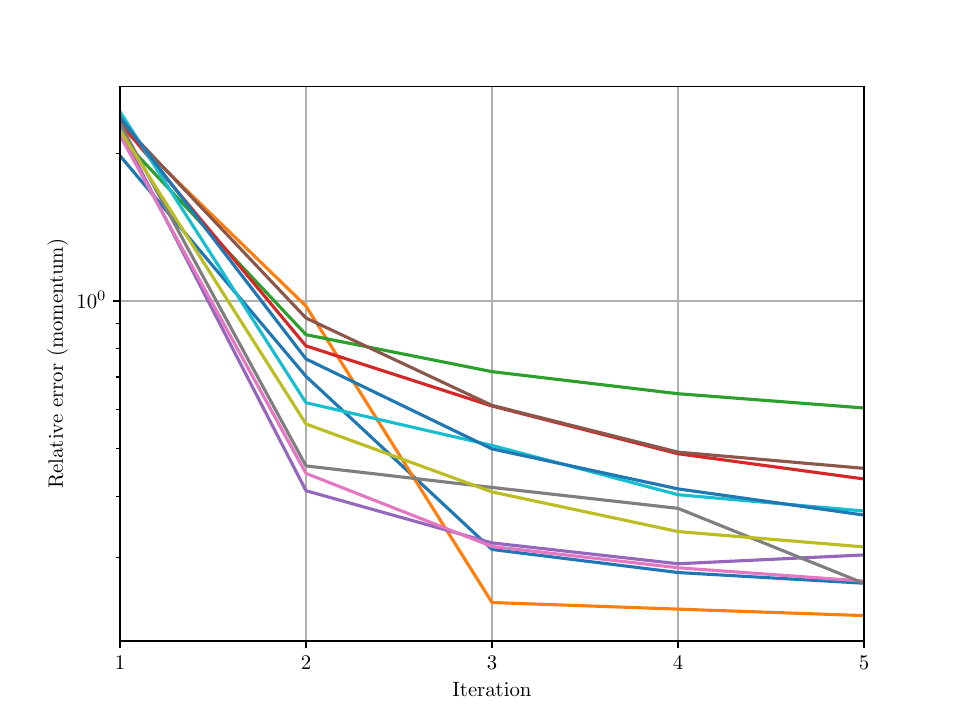}
    \includegraphics[scale=0.20]{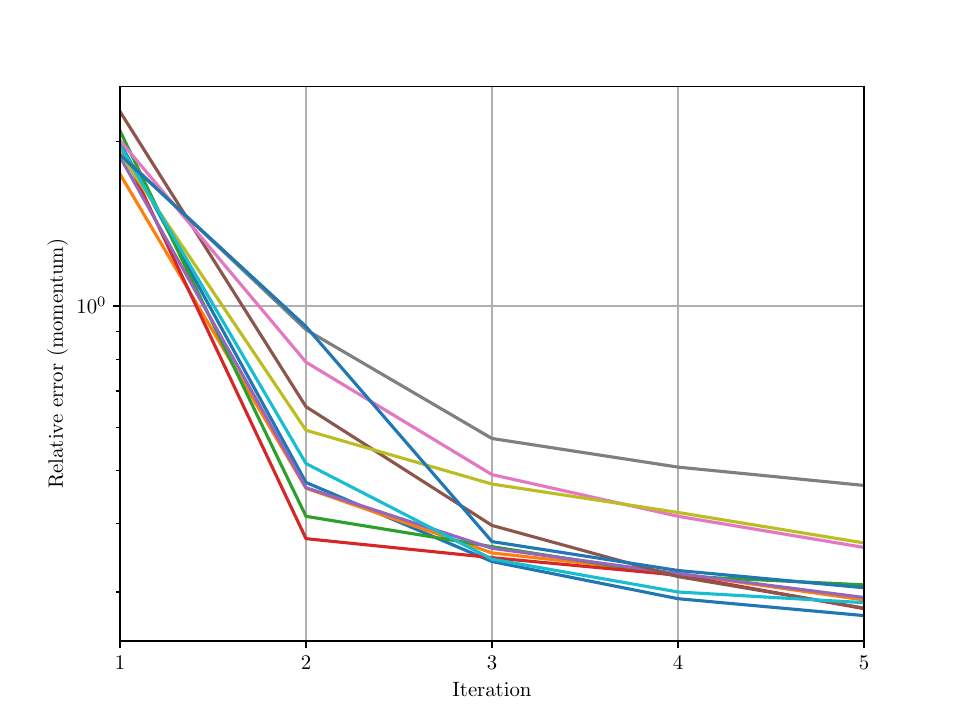}\\
    \includegraphics[scale=0.20]{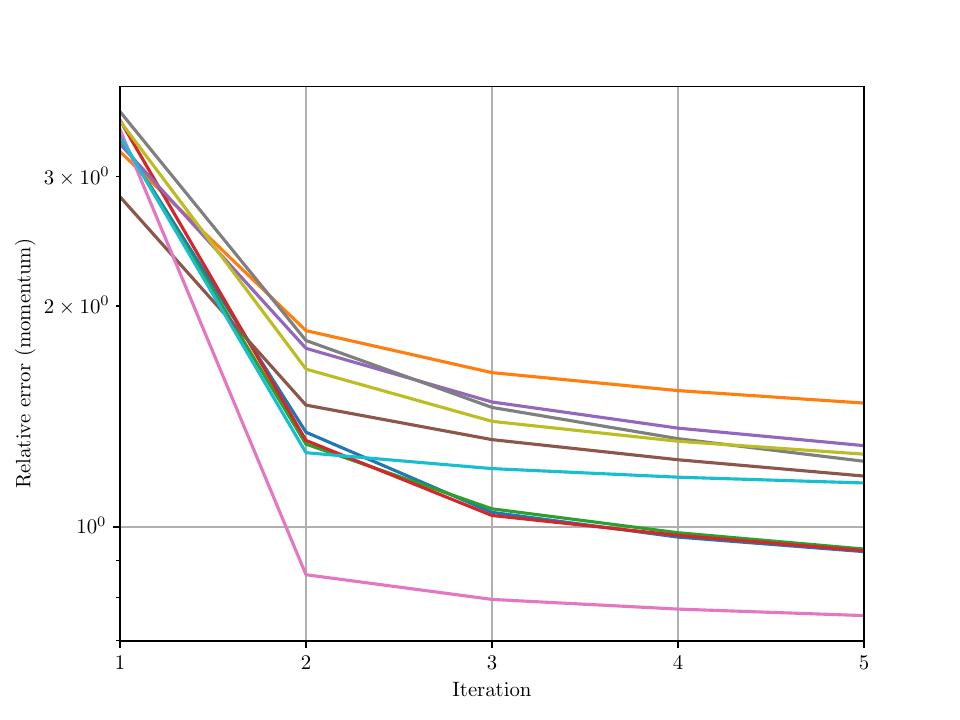}
    \includegraphics[scale=0.20]{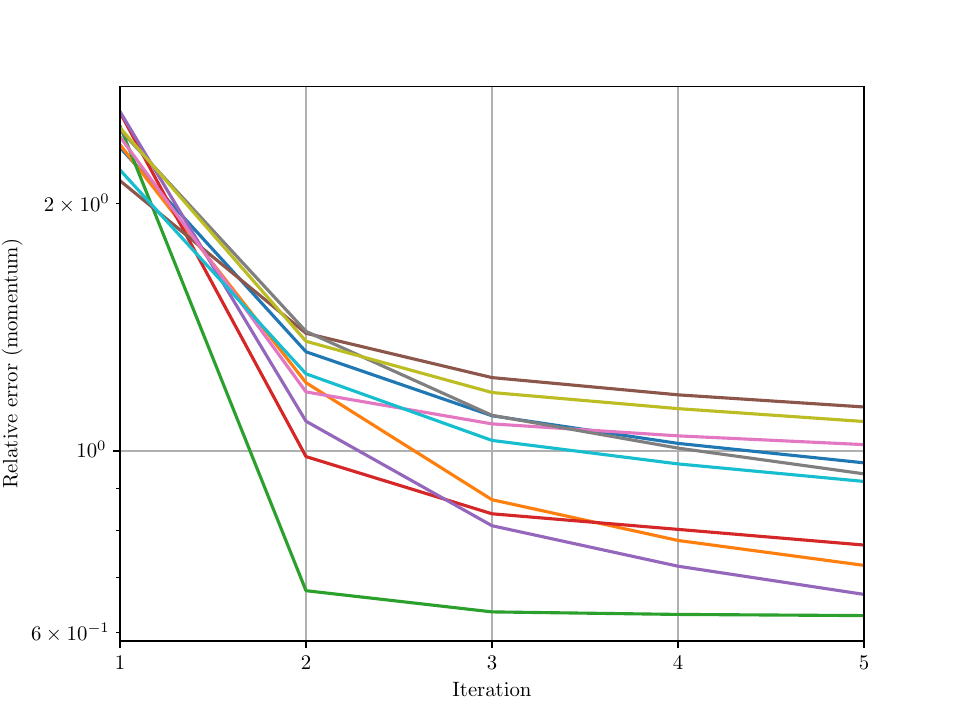}
    \includegraphics[scale=0.20]{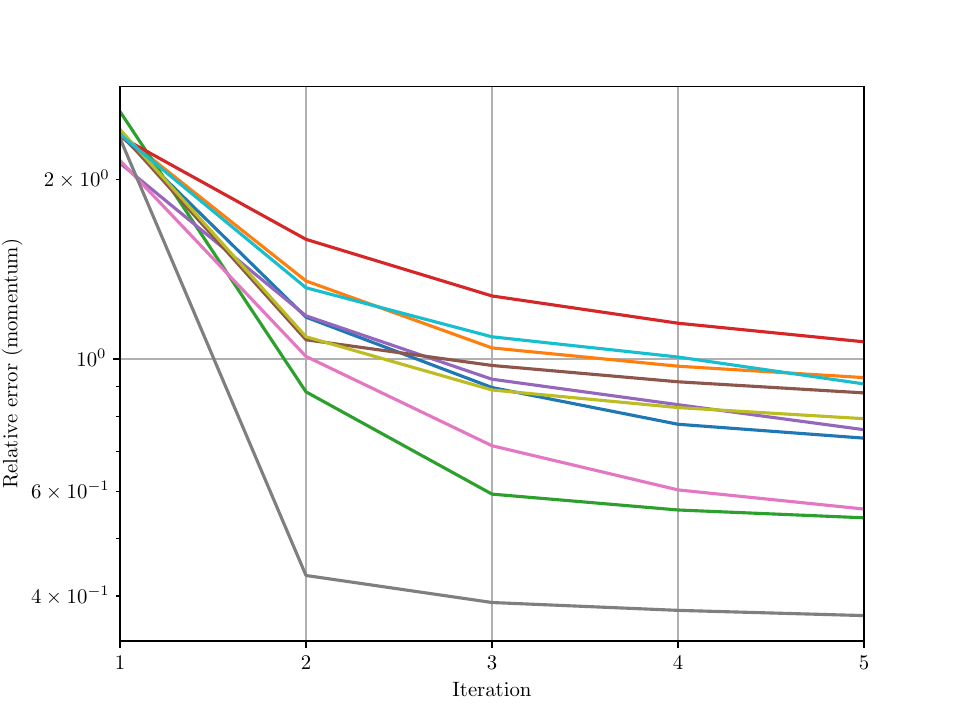}
    \includegraphics[scale=0.20]{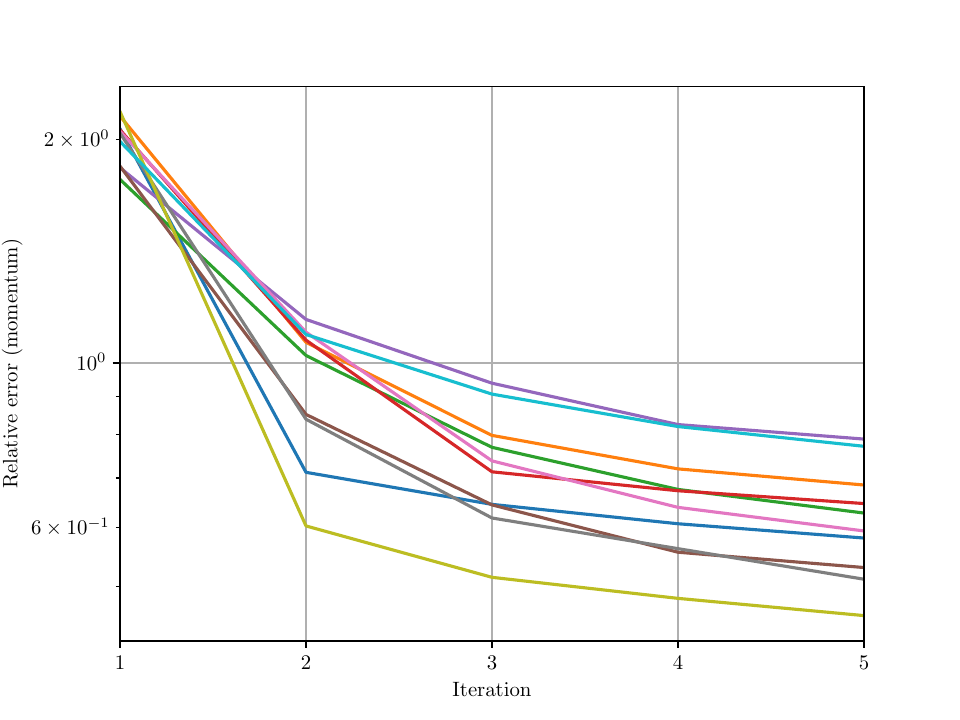}
    \caption{Relative error (equation \eqref{eq:relative_error}). The rows corresponds
    to $N=20$, $N=40$ and $N=80$ and the columns to the targets in figure \ref{fig:synthetic_targets}.}
    \label{fig:relative_error}
\end{figure}

\begin{figure}
    \centering
    \includegraphics[scale=0.20]{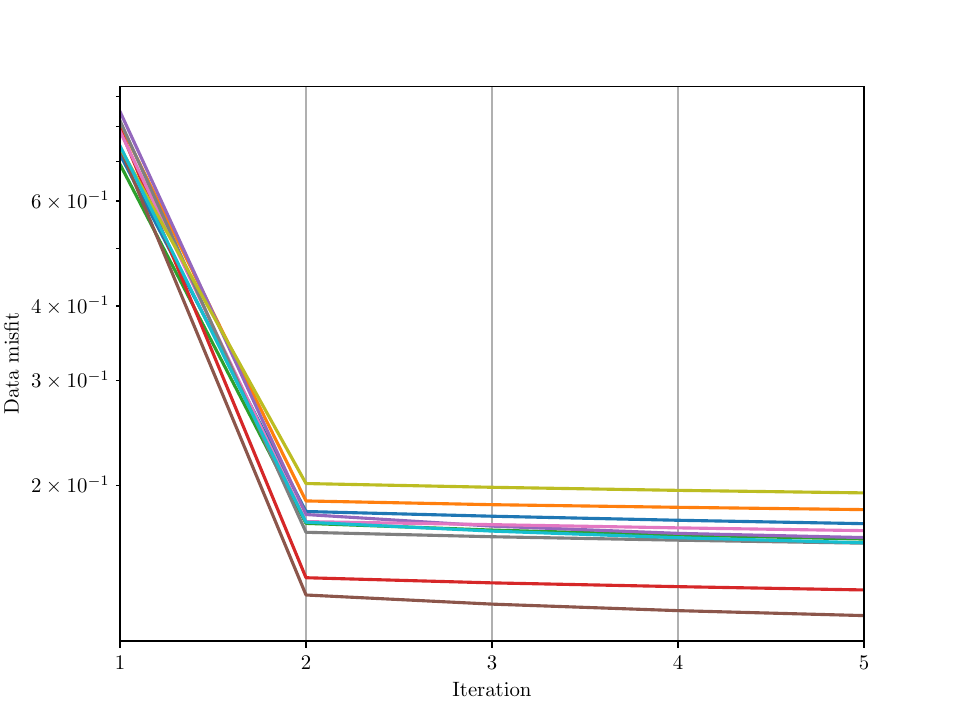}
    \includegraphics[scale=0.20]{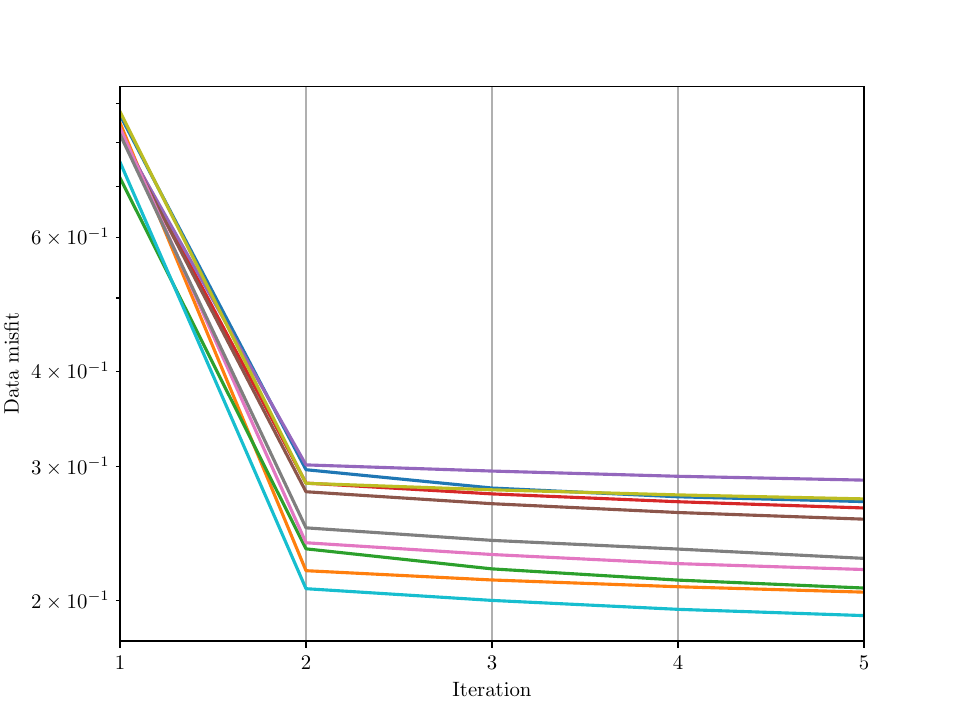}
    \includegraphics[scale=0.20]{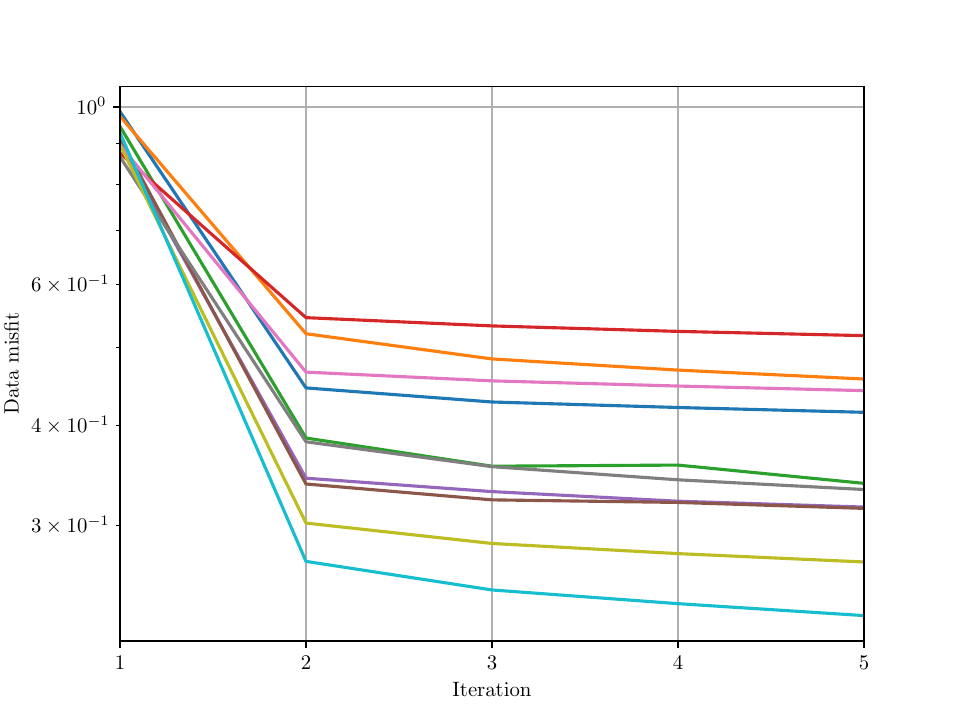}
    \includegraphics[scale=0.20]{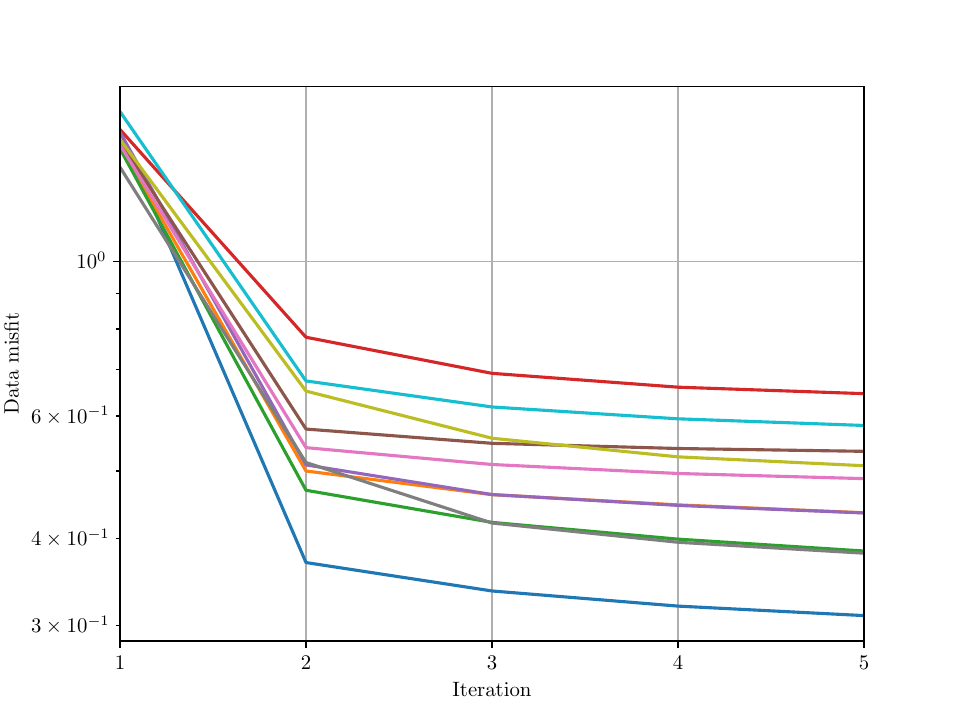}\\
    \includegraphics[scale=0.20]{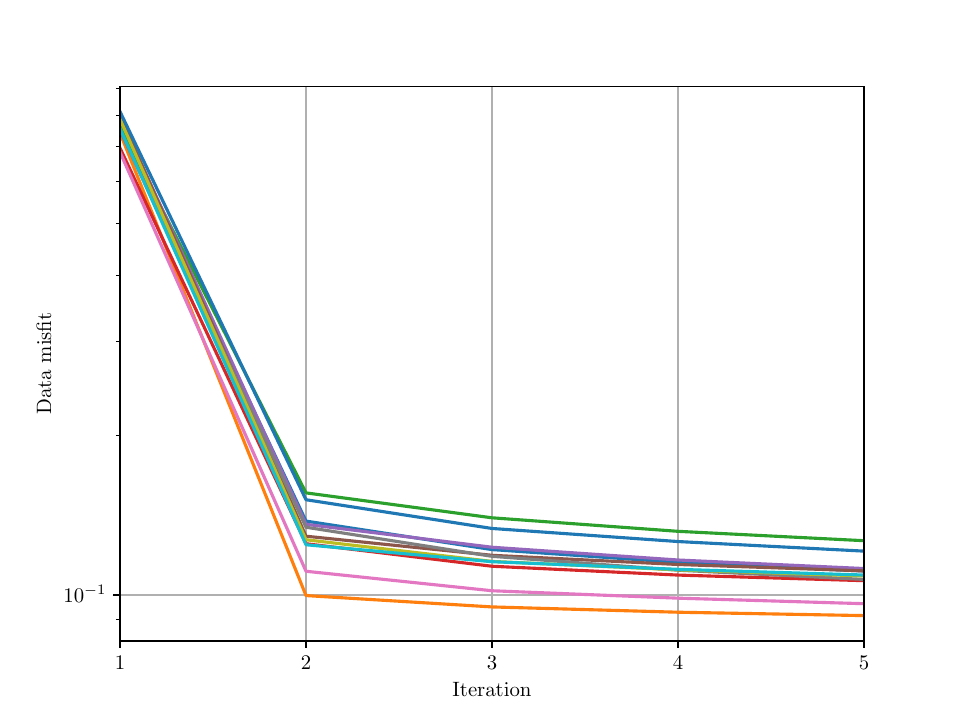}
    \includegraphics[scale=0.20]{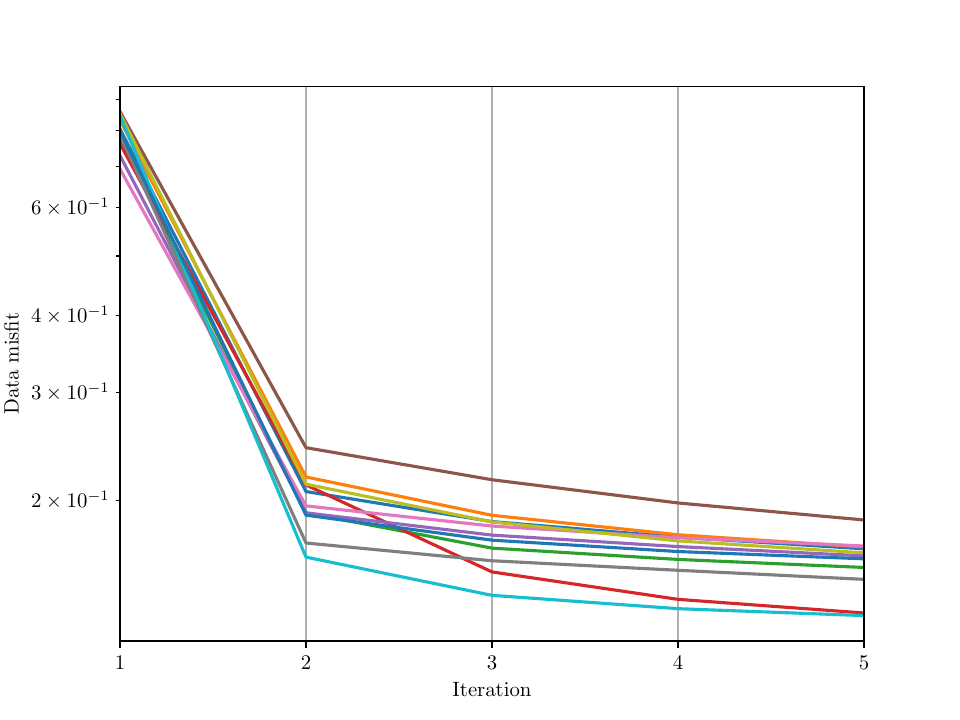}
    \includegraphics[scale=0.20]{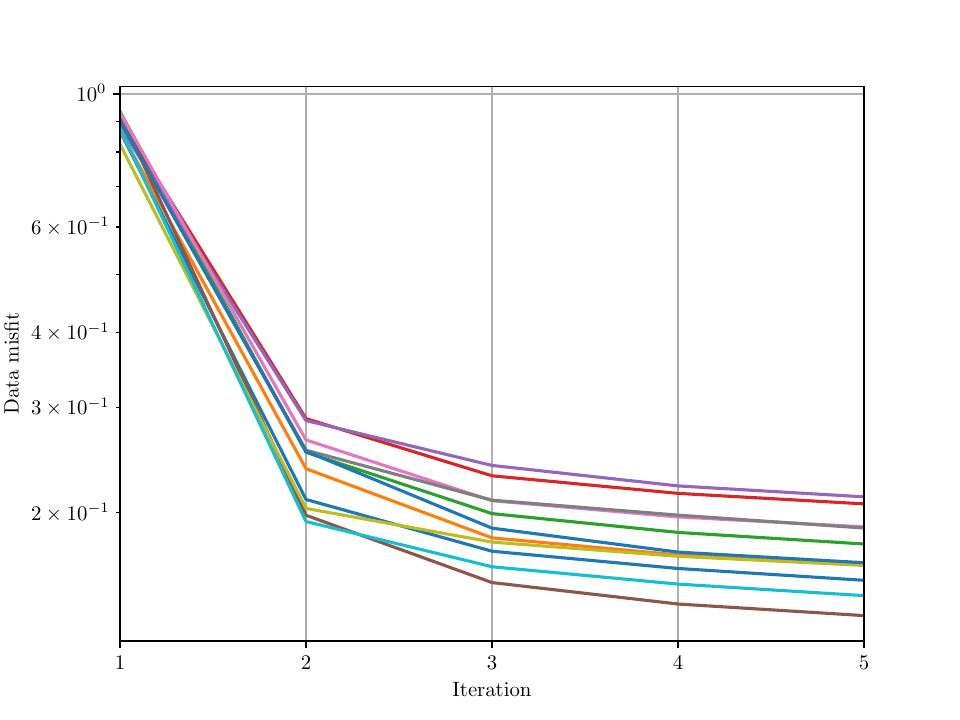}
    \includegraphics[scale=0.20]{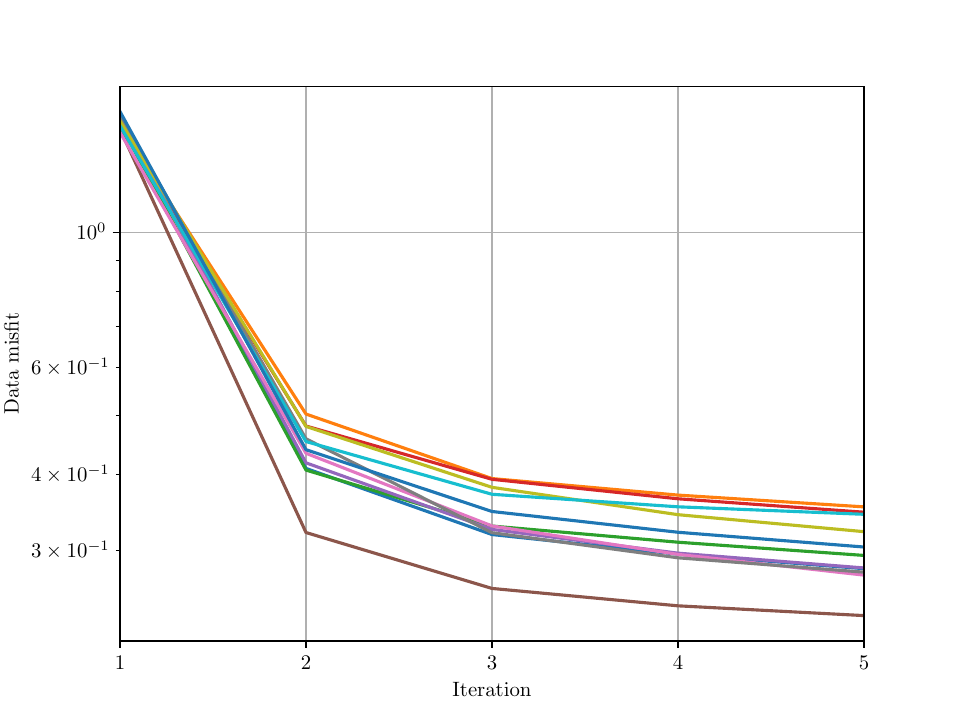}\\
    \includegraphics[scale=0.20]{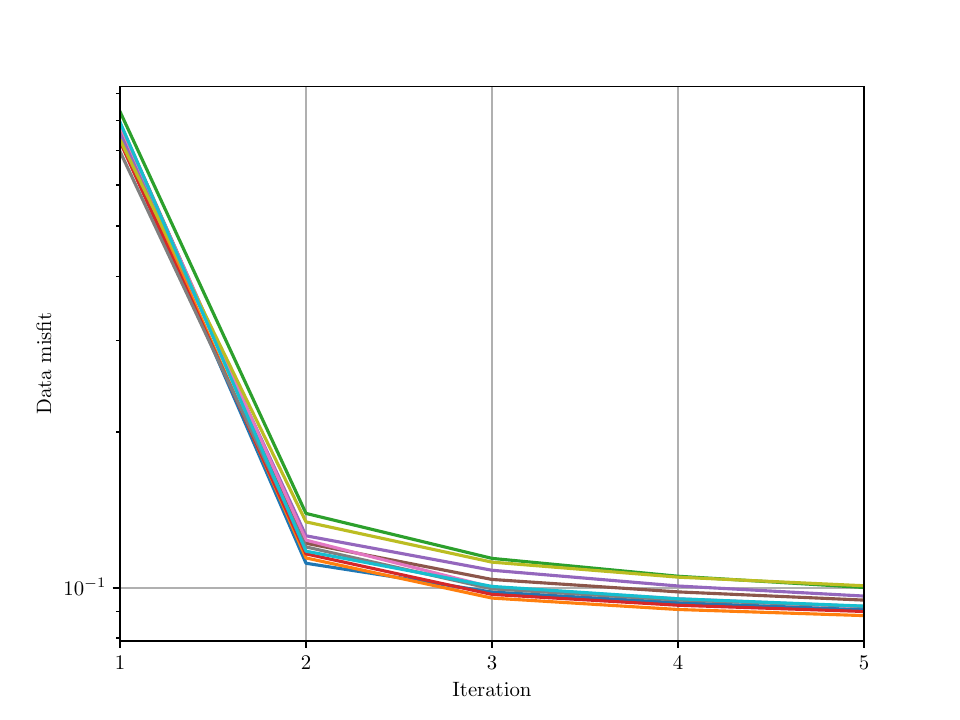}
    \includegraphics[scale=0.20]{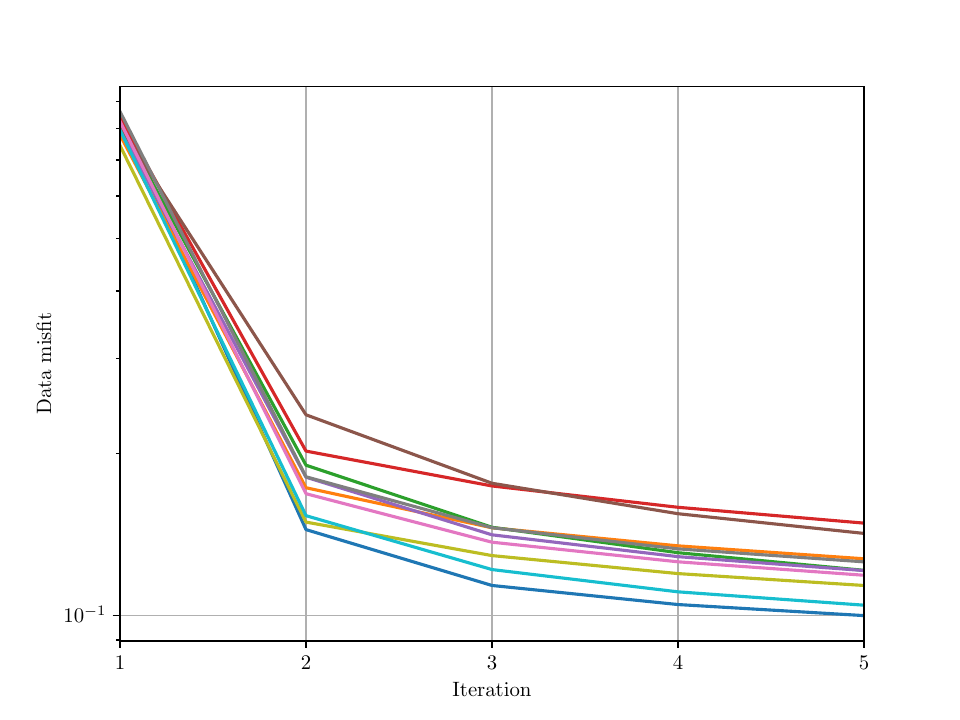}
    \includegraphics[scale=0.20]{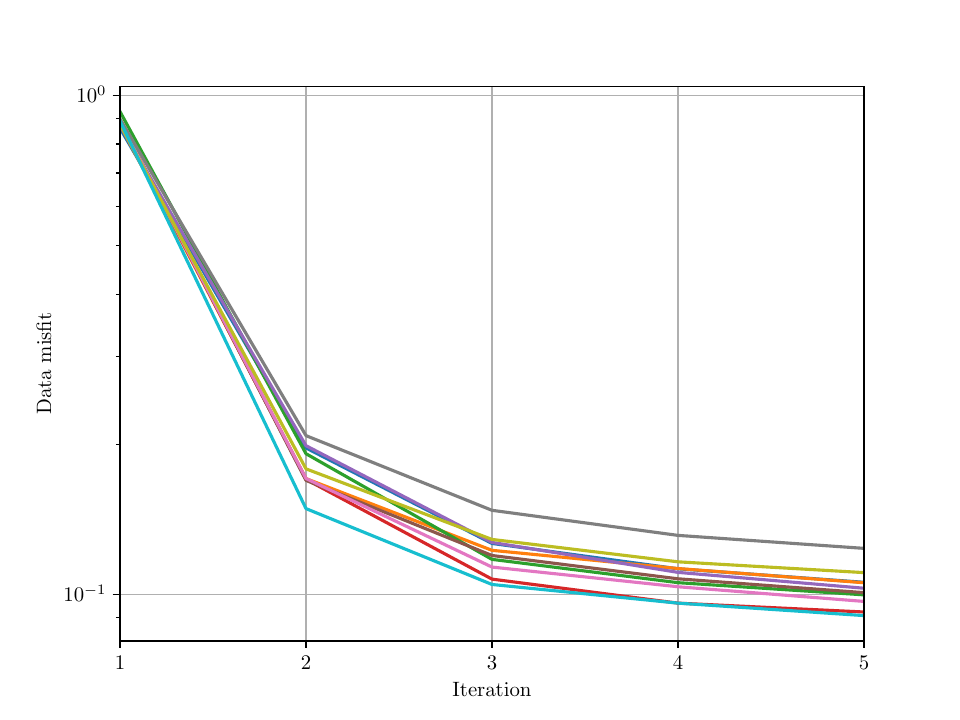}
    \includegraphics[scale=0.20]{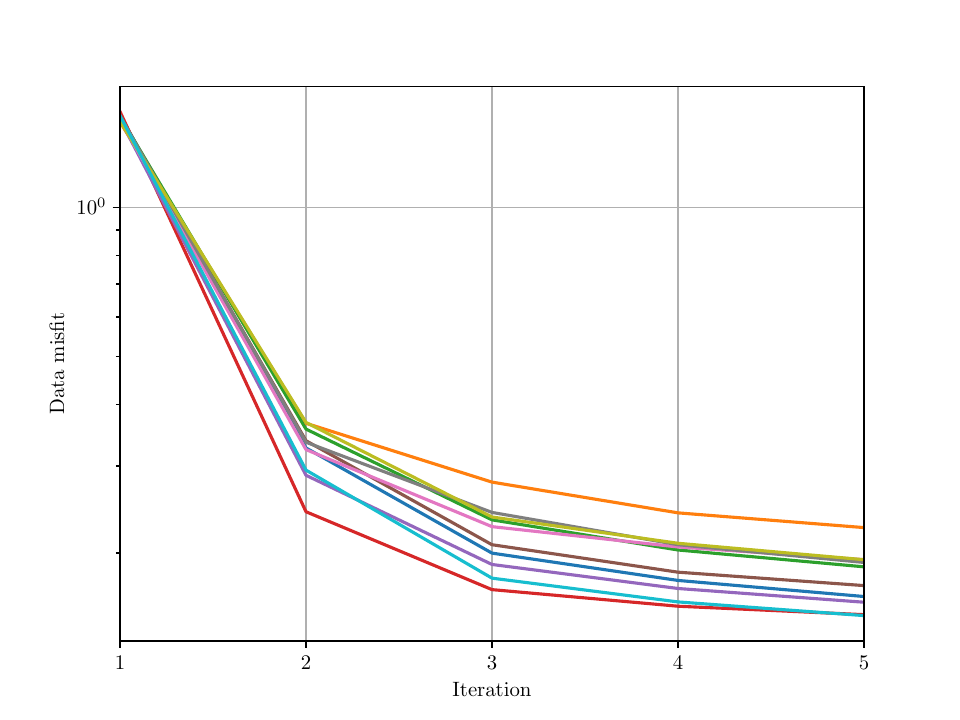}
    \caption{Data misfits  (equation \eqref{eq:mismatch:C}). The rows corresponds
    to $N=20$, $N=40$ and $N=80$ and the columns to the targets in figure
     \ref{fig:synthetic_targets}.}
    \label{fig:data_misfits}
\end{figure}

\begin{figure}
    \centering
    \includegraphics[scale=0.20]{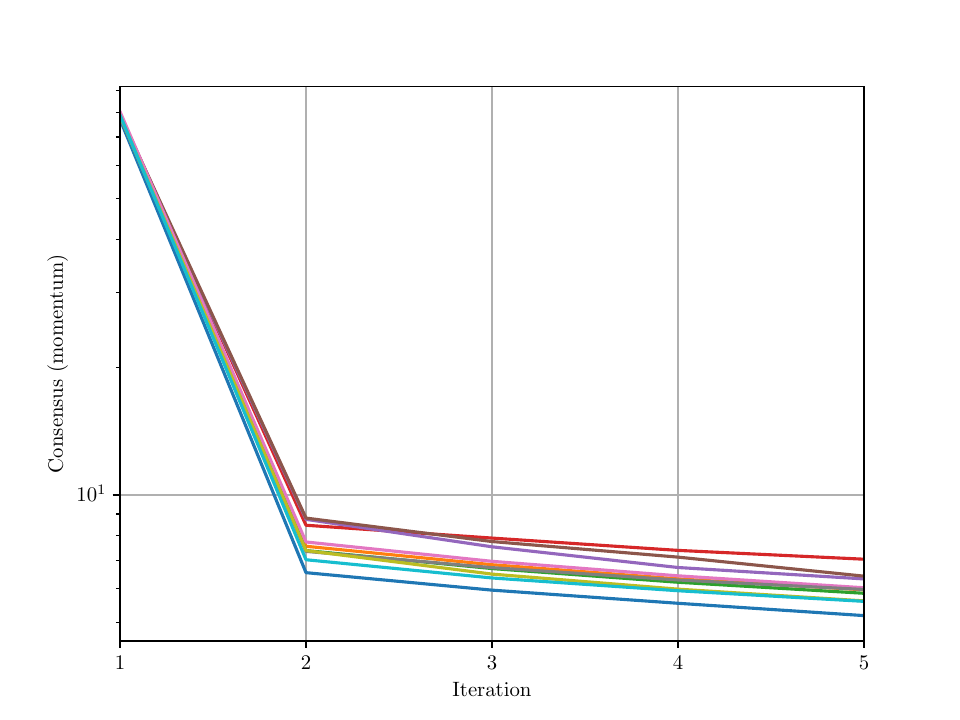}
    \includegraphics[scale=0.20]{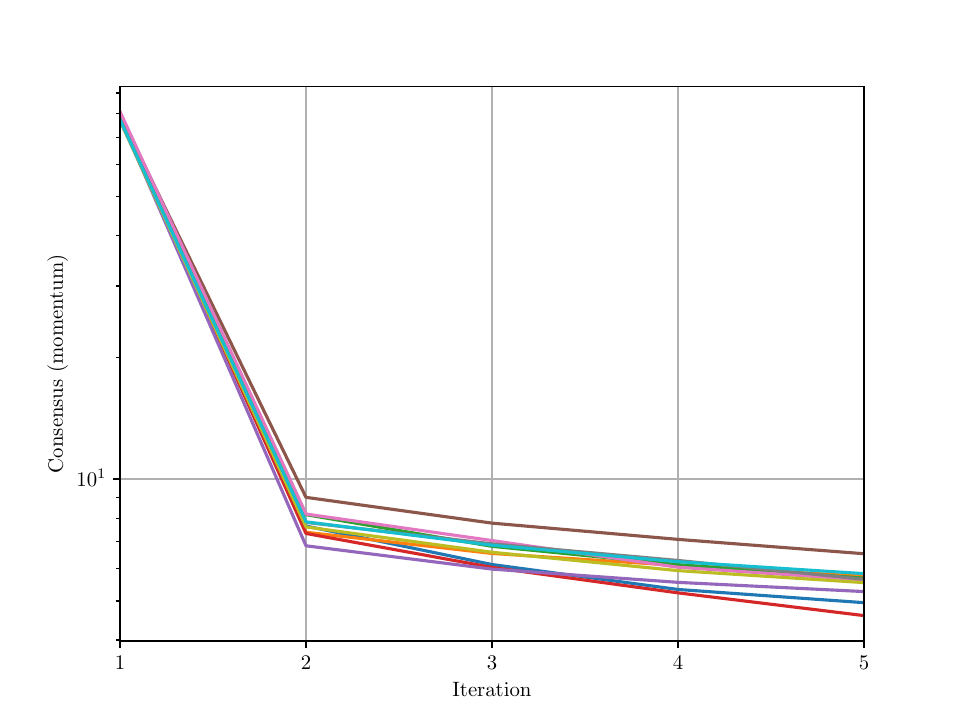}
    \includegraphics[scale=0.20]{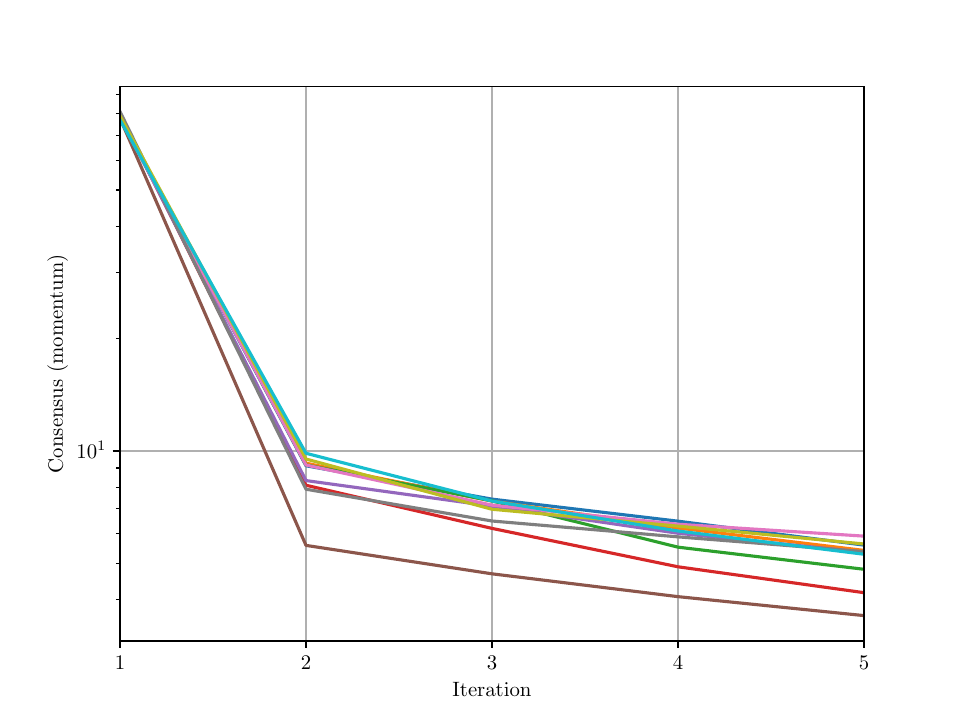}
    \includegraphics[scale=0.20]{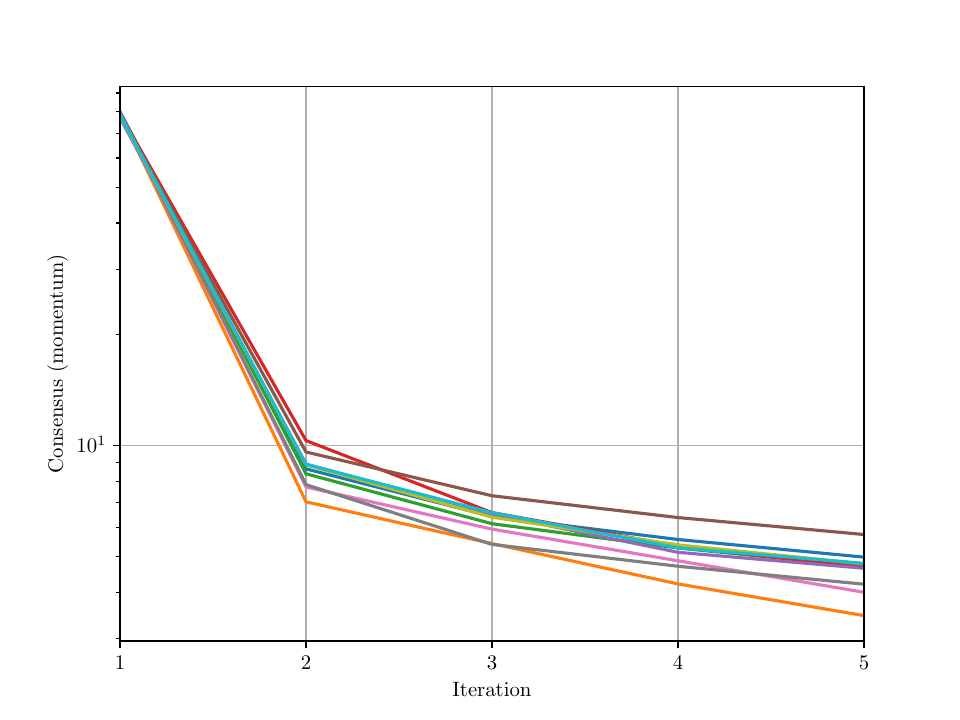}\\
    \includegraphics[scale=0.20]{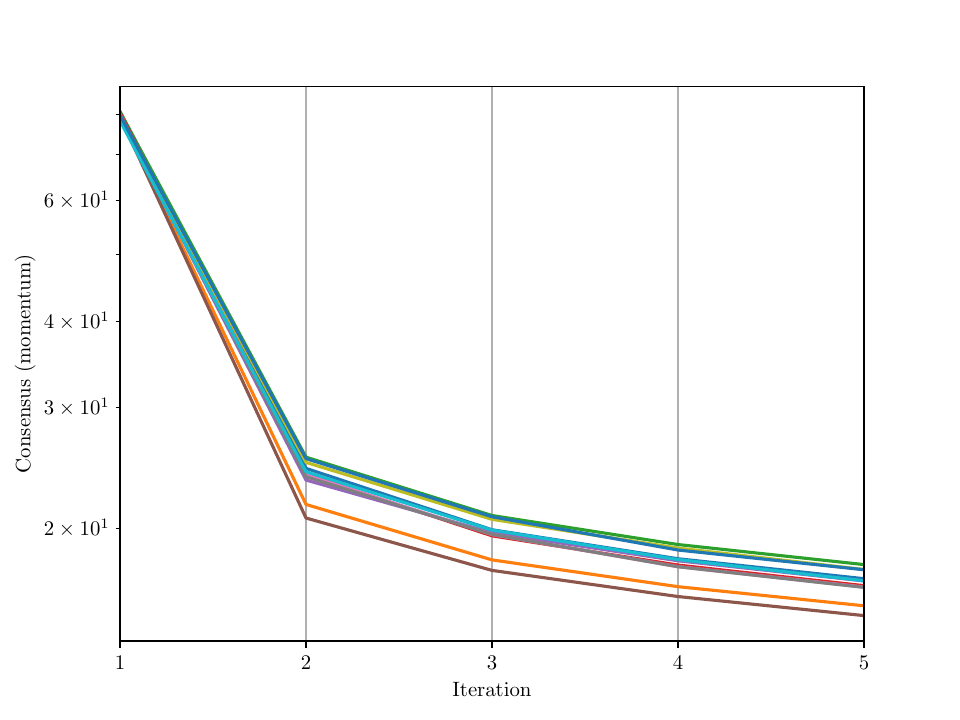}
    \includegraphics[scale=0.20]{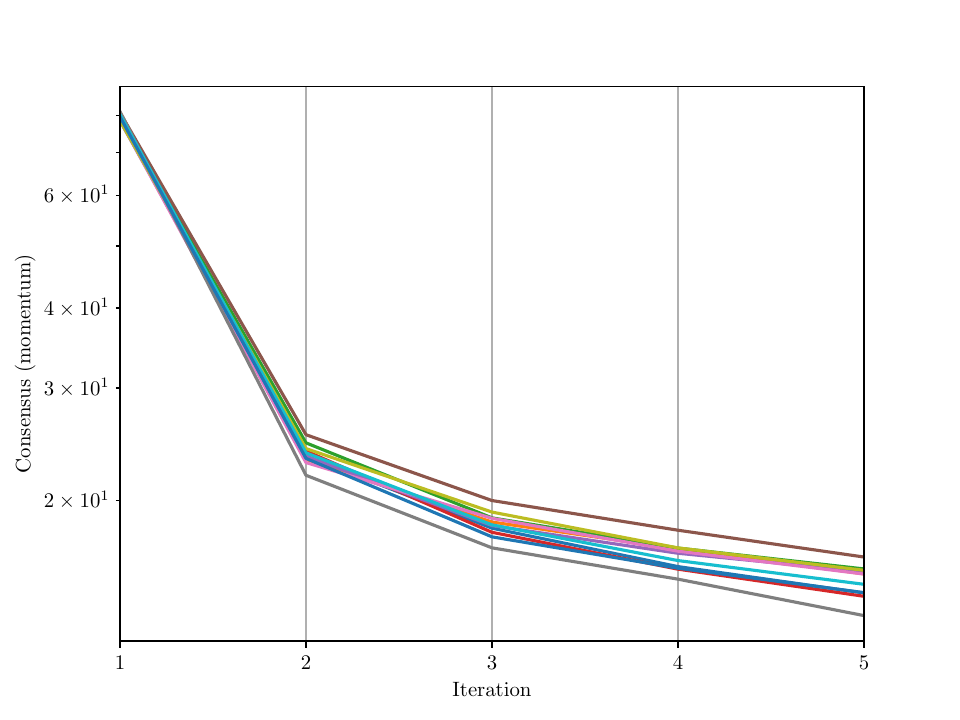}
    \includegraphics[scale=0.20]{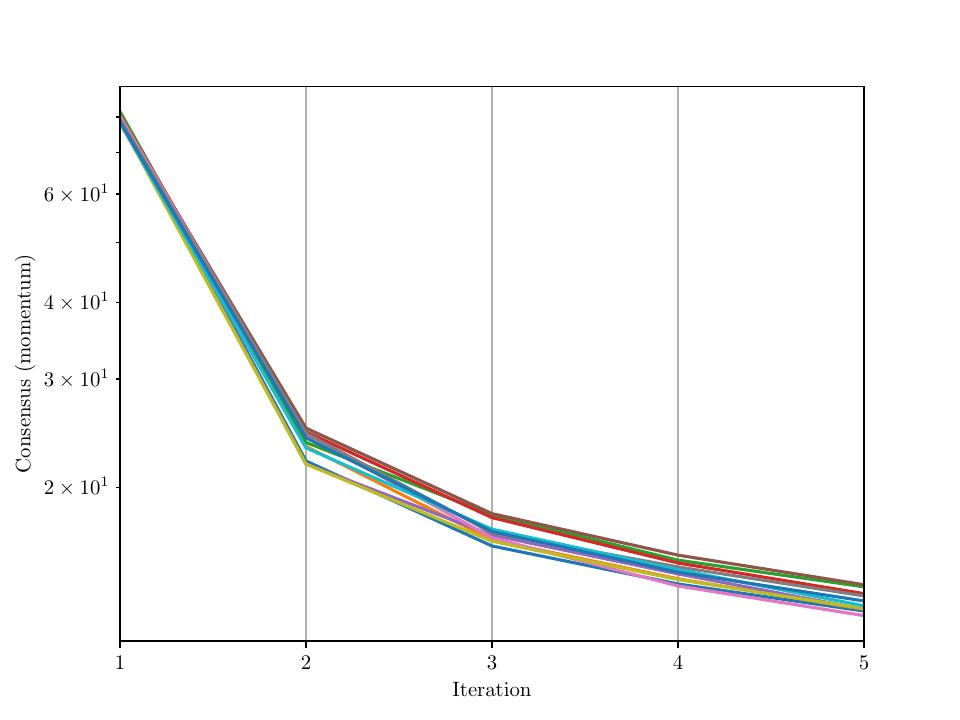}
    \includegraphics[scale=0.20]{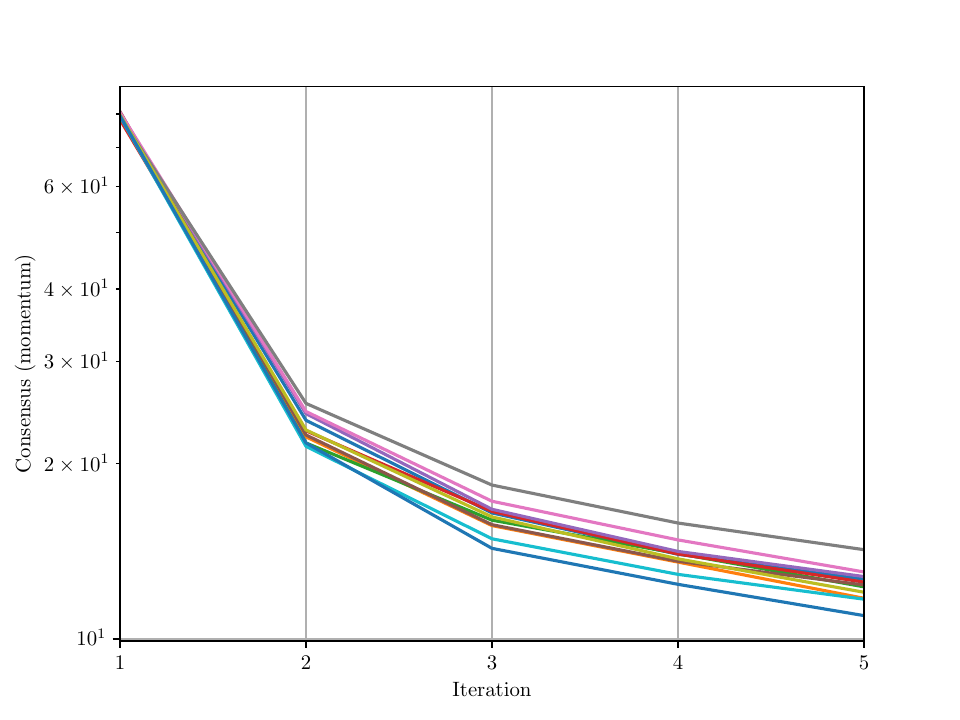}\\
    \includegraphics[scale=0.20]{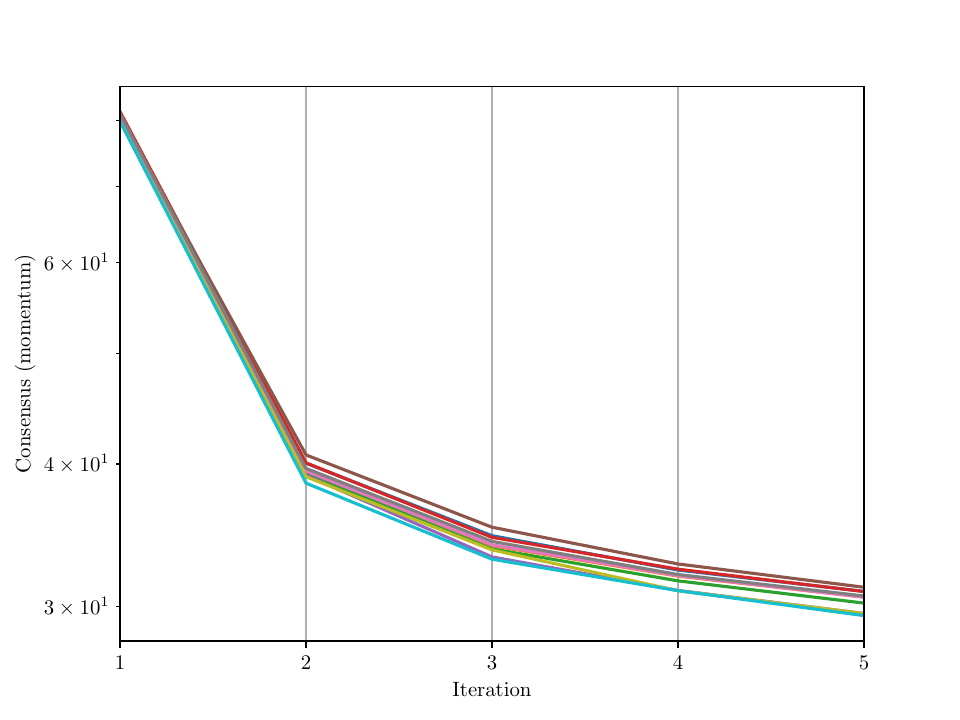}
    \includegraphics[scale=0.20]{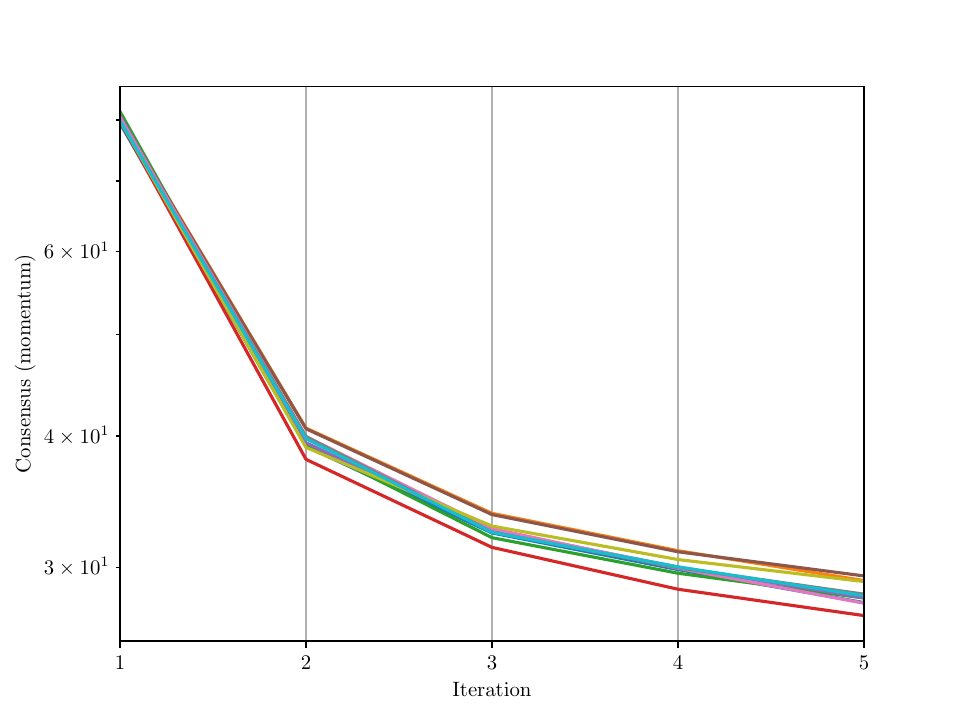}
    \includegraphics[scale=0.20]{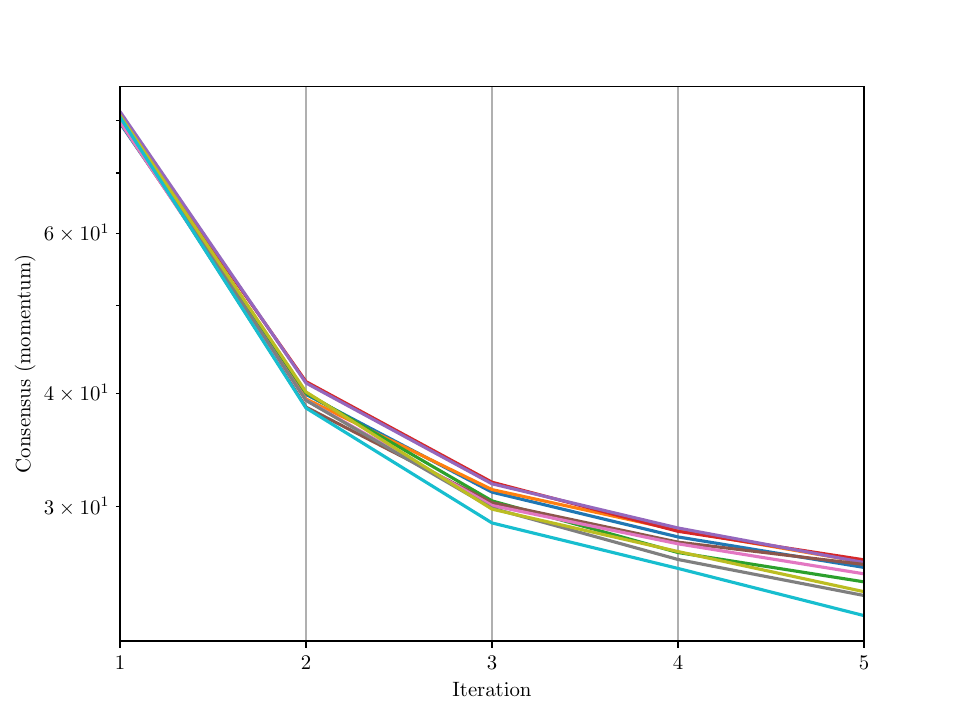}
    \includegraphics[scale=0.20]{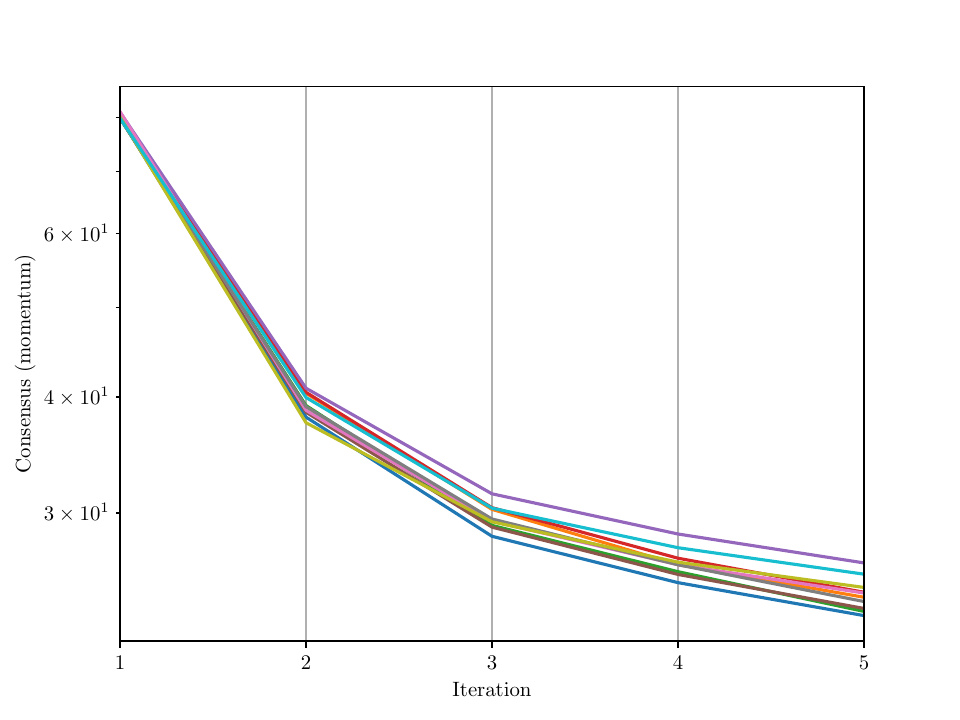}
    \caption{Momentum consensus deviation (equation \eqref{eq:consensus}). The rows corresponds
    to $N=20$, $N=40$ and $N=80$ and the columns to the targets in figure
     \ref{fig:synthetic_targets}.}
    \label{fig:consensus}
\end{figure}

%% file: conclusion.tex
\section{Summary and outlook}\label{sec:conclusion}

In this paper we have presented a parameterisation- and derivative-free
method for matching closed planar curves. A moving mesh discretisation
of Hamilton's equations for curves was described
using the induced diffeomorphism of a vector field occupying
the Wu-Xu finite element space. We also describe a
transformation theory for this element facilitating a computationally 
performant forward model for use in the associated inverse problem.
Finding the momentum encoding the
forward motion of the template matches a desired curve was treated
as a Bayesian inverse problem in section \ref{sec:inverse_problem} and
EKI was used to approximate its solution. The numerical results presented in
section \ref{sec:applications} suggests that the method shows great
promise. Not only does is it easy to implement, the EKI is shown to
quickly reach ensemble consensus meaning that it is efficient in exploring
the subspace spanned by the initial ensemble. This is in part
thanks to the momentum being a one-dimensional signal on the template.
Treating the mismatch term in a negative Sobolev norm was shown to increase
both accuracy of our results and robustness to mesh resolution. We also
showed that the method is
robust to the choice of initial ensemble even when the ensemble size is
less than half the dimension of the forward problem. Further, assuming
the forward operator is scalable as the mesh is refined (large-scale PDE
solves are common in many areas of scientific computing, and the inverse
needed in the Kalman gain scales cubically in $N$ \cite{mandel2006efficient}).\\

Future work includes proving convergence of the finite element 
discretisation for \eqref{eq:hamiltonseqs_reduced:disc} and 
subsequently using these error estimates to quantify error in a rigorous
treatment of the Bayesian inverse problem
\cite{cotter2010approximation}. As indicated in \cite{bock2021note},
some challenges exist for nonconforming finite element methods with
singular source terms.
The template considered in this paper is a piece-wise linear curve. An
obvious extension would be to apply isoparametric methods to cater for
piece-wise higher-order polynomial curves. The effect of this would only
affect the right-hand side and would not affect regularity results for 
the velocity. An advantage of the finite element method for curves is also
that it allows for adaptivity e.g. refinement of the mesh only 
in the vicinity of the embedded template.
We considered problems of modest size to illustrate the discretisation
and the EKI. As the mesh is refined, it is likely the case that the
dimension of the forward operator dwarfs the size of the ensemble and effects
of the MC approximation are more pronounced. This is the case for ensemble methods
for e.g. numerical weather prediction and several techniques exist to counter
these effects \cite{petrie2008localization} (e.g. localisation or covariance
inflation). In particular, localisation methods may be suitable to assume conditional
independence between separated states (i.e. parts of the shape that are distant
in physical space) so as to counter spurious correlations.

%% file: appendix.tex
\section{Proof of Theorem \ref{thm:momchar}}\label{app:momchar}
The momentum satisfies
\[
\dot{p}_t + \nabla u_t\transp\circ q_t p_t = 0
\]
Using the ansatz we verify:
\begin{align*}
\dot{p}_t + \nabla u_t\transp\circ q_t p_t & = \dot{J_t\invtransp} p_0 + \nabla u_t\transp
J_t\invtransp p_0\\
& = - J_t\invtransp d(J_t\transp) J_t\invtransp  p_0 + \nabla u_t\transp\circ q_t J_t\invtransp p_0\\
& = - J_t\invtransp (d J_t)\transp J_t\invtransp  p_0 + \nabla u_t\transp\circ q_t J_t\invtransp p_0\\
& = - J_t\invtransp (\nabla u_t\circ q_t J_t)\transp J_t\invtransp  p_0 + \nabla u_t\transp\circ q_t J_t\invtransp p_0\\
& = - J_t\invtransp J_t\transp \nabla u_t\transp\circ q_t J_t\invtransp  p_0 + \nabla u_t\transp\circ q_t J_t\invtransp p_0\\
& = - \nabla u_t\transp J_t\invtransp  p_0 + \nabla u_t\transp\circ q_t J_t\invtransp p_0\\
& = 0\,.
\end{align*}